\documentclass{article}

\usepackage[main]{neurips_2026}

\usepackage{graphicx} 
\usepackage{xcolor}
\usepackage{amsmath,amssymb,amsthm,mathtools}
\usepackage{thmtools}
\usepackage{enumitem}
\usepackage{algorithm}
\usepackage[noend]{algpseudocode}
\usepackage{microtype}
\usepackage[hidelinks]{hyperref}
\usepackage{wrapfig}
\usepackage[nameinlink,capitalize]{cleveref}

\newtheorem{theorem}{Theorem}[section]

\newtheorem{lemma}[theorem]{Lemma}
\newtheorem{claim}[theorem]{Claim}
\newtheorem{corollary}[theorem]{Corollary}
\newtheorem{remark}[theorem]{Remark}
\newtheorem{definition}[theorem]{Definition}

\crefname{algorithm}{Algorithm}{Algorithms}
\Crefname{algorithm}{Algorithm}{Algorithms}
\floatname{algorithm}{Algorithm}
\algrenewcommand\algorithmiccomment[1]{\hfill {\footnotesize$\triangleright$ #1}}

\newcommand{\R}{\mathbb{R}}
\newcommand{\E}{\mathbb{E}}
\newcommand{\Z}{\mathbb{Z}}
\newcommand{\poly}{\mathrm{poly}}
\newcommand{\N}{\mathcal{N}}
\newcommand{\calS}{\mathcal{S}}
\newcommand{\Prb}{\mathbb{P}}
\newcommand{\eps}{\varepsilon}

\newcommand{\idx}{\mathsf{idx}}
\newcommand{\Unif}{\mathrm{Unif}}
\newcommand{\diag}{\mathrm{diag}}
\newcommand{\dd}{\,\mathrm{d}}

\makeatletter
\renewcommand{\@noticestring}{\@trackname}
\makeatother

\newcommand{\quant}{\ensuremath{\mathsf{quant}}}

\newcommand{\norm}[1]{\left\|#1\right\|}
\newcommand{\ip}[2]{\left\langle #1,#2 \right\rangle}

\usepackage{tikz}
\usepackage{pgfplots}
\usetikzlibrary{decorations.pathreplacing, positioning}

\pgfplotsset{compat=1.17}

\title{
DASHQuant:
Provable Quantization with a Single Randomized Hadamard Transform}

\date{}

\author{%
  Ying Feng \\
  MIT \\
  \texttt{yingfeng@mit.edu} \\
  \And
  Piotr Indyk \\
  MIT \\
   \texttt{indyk@mit.edu}\\
  \And
  Michael Kapralov \\
  EPFL \\
  \texttt{michael.kapralov@epfl.ch}\\
  \And
  Dmitry Krachun \\
  Princeton University \\
  \texttt{dk9781@princeton.edu}\\
  \And
  Boris Prokhorov \\
  EPFL \\
  \texttt{boris.prokhorov@epfl.ch}\\
}

\begin{document}

\maketitle

\begin{abstract}
Vector quantization via random projection followed by scalar quantization is a fundamental primitive in machine learning, with applications ranging from similarity search to federated learning and KV cache compression. While dense random rotations yield clean theoretical guarantees, they require $\Theta(d^2)$ time. The randomized Hadamard transform $HD$ reduces this cost to $O(d \log d)$, but its discrete structure complicates analysis and leads to weaker or purely empirical compression guarantees. 
In this work, we introduce DASHQuant  (\textbf{D}ithered \textbf{A}ccelerated \textbf{S}ingle-\textbf{H}adamard Quantization), a provable variant of this approach.  Specifically, the quantizer applies $HD$ to the input vector and adds a random dithering before quantizing, injecting additional randomness at negligible cost.   We prove that this approach is unbiased and provides mean squared error bounds that asymptotically match those achievable with
truly random rotation matrices. In particular, we prove that a dithered version of TurboQuant, and a special case of EDEN,  achieves mean squared error $\bigl(\pi\sqrt{3}/2 + o(1)\bigr) \cdot 4^{-b}$ at $b$ bits per coordinate, where the $o(1)$ term vanishes uniformly over all unit vectors and all dimensions as the number of quantization levels grows. 
\end{abstract}

\section{Introduction}
 
Vector quantization is a fundamental technique for compressing real-valued vectors in $\R^d$ into compact binary representations. Given a vector $x \in \R^d$, a quantizer $\textsc{VectorQuant}$ maps $x$ to a sequence of $b\cdot d$ bits $\textsc{VectorQuant}(x)$. The bits can then be input to another procedure $\textsc{VectorDequant}$, obtaining the decoded vector $\widetilde{x}=\textsc{VectorDequant}(\textsc{VectorQuant}(x))$. The size $b\cdot d$ of the compressed representation  controls the distortion $\|x-\widetilde{x}\|^2$, or the expected distortion $E[\|x-\widetilde{x}\|^2]$ if the procedures are randomized. Alternatively, instead of decoding $\widetilde{x}$, some function of it (e.g., the distance to another vector $y$) is estimated. 
This primitive plays an important role in many areas of machine learning, including similarity search~\cite{charikar2002similarity, jegou2010product, datar2004locality, gao2024rabitq,gao2025practical}, federated learning~\cite{vargaftik2022eden, vargaftik2021drive, ben2024accelerating}, efficient transformers~\cite{zandieh2025turboquant}, and model compression~\cite{gholami2022survey,ashkboos2024quarot,tseng2024quip}.
 
A popular and well-studied approach to vector quantization is {\em random projection followed by scalar quantization} \cite{charikar2002similarity, datar2004locality,dong2008asymmetric,andoni2015practical,vargaftik2022eden,vargaftik2021drive,gao2024rabitq,ashkboos2024quarot,tseng2024quip,gao2025practical,zandieh2025turboquant}. Here, one computes $\textsc{VectorQuant}(x) = \quant(Rx)$, where $R$ is a random rotation matrix and $\quant$ is a scalar quantizer (e.g., the sign function) applied coordinate-wise to the vector $Rx$. A key advantage of this approach is that it is {\em data-oblivious}: it requires no complex preprocessing tailored to the input distribution. Moreover, this method is amenable to theoretical analysis, since each coordinate of $Rx$ is approximately normally distributed when $R$ is a uniformly random rotation. 
 
However, applying a dense random rotation matrix $R$ requires $\Theta(d^2)$ time, which can be prohibitively expensive for high-dimensional vectors. To overcome this bottleneck, \cite{ailon2009fast, sarlos2006improved,woolfe2008fast,krahmer2011new,halko2011finding} and others proposed to replace random projections with {\em randomized Hadamard transforms}. Specifically, one replaces $R$ with the structured matrix $HD$, where $H \in \R^{d \times d}$ is the normalized Hadamard matrix and $D$ is a diagonal matrix with independent Rademacher ($\pm 1$) entries on the diagonal. The resulting transform can be applied in $O(d \log d)$ time using the fast Walsh-Hadamard algorithm, yielding a substantial speedup over dense rotations. Thanks to its low complexity and efficient parallel implementation, randomized Hadamard transform is often the method of choice in industrial deployments, see, e.g.,~\cite{christiani2025rotational,hadacore2024blog}.

Unfortunately, $HD$ is no longer a uniformly random rotation; in fact, it is not even a continuous random variable. As a consequence, the coordinates of $HDx$ have a more complex distribution than in the dense case, and theoretical guarantees for quantization become harder to establish. Existing results are typically either empirical~\cite{andoni2015practical, suresh2017distributed,ordentlich2025optimal} or provable but with higher dimension or number of bits compared to those achievable with truly random rotations~\cite{ailon2013almost,krahmer2011new,tropp2011improved,le2014fastfood,cherapanamjeri2022uniform,ben2024accelerating}. A natural remedy is to use more general {\em structured matrices} obtained by composing multiple randomized Hadamard transforms, possibly interleaved with other (e.g., sparse diagonal) matrices~\cite{dasgupta2011fast,le2014fastfood, zhangfaster, kennedy2016fast,feng2026fast,zilca2026approximating,BenBasatEtAl2026}. 
In particular, the last two papers independently show that the composition of just {\em two} randomized Hadamard transforms suffices to approximate
the marginal distributions of a uniform random rotation
in high dimensions.\footnote{They show that, for each vector $x$ and coordinate $i$,  $(HD_1 HD_2 x)_i$ approximates $(Rx)_i$ in Kolmogorov and
Wasserstein-1 distances, up to an error of
$\widetilde{O}(d^{-1/2})$. See Related Work for further discussion.}
However, those matrices are still more complex than a single
randomized transform. Furthermore, the approximation error
translates into an increased bound for the expected distortion
after quantization.


 
In this paper, we study a simple alternative to composing multiple randomized Hadamard transforms, namely {\em dithering}. We propose DASHQuant:  \textbf{D}ithered \textbf{A}ccelerated \textbf{S}ingle-\textbf{H}adamard Quantization. Specifically, each coordinate of $HDx$ is shifted by a random, non-uniform amount (controlled by a parameter $U$) before applying the scalar quantizer $\quant$.
Such quantizers are quite simple, as they require only one application of the randomized Hadamard transform, followed by a subtraction of a random vector. At the same time,  the random offset injects additional randomness into the quantization process, which facilitates the analysis.
\subsection{Our results}

We consider the following data-oblivious vector quantization problem.  For an input vector
$x\in \R^d$, a quantizer $\textsc{VectorQuant}:\R^d \rightarrow \{0, 1\}^{b\cdot d}$ outputs a binary string of $bd$ bits, where $b$ is the bit-width per coordinate.  The dequantizer $\textsc{VectorDequant}:\{0, 1\}^{b\cdot d}\rightarrow \R^d$ maps this string back to a vector $\widetilde x\in \R^d$. We aim to design a pair of quantizer and dequantizer such that for any desired bit-width $b$ minimize
the following expected distortion measure for any (worst-case) vectors:
\begin{equation*}
    \sup_{x\in \calS^{d-1}} \E\|x-\widetilde x\|_2^2,
\end{equation*}
where the expectation is over the internal randomness of the quantizer.  The restriction to unit vectors is only for normalization: for nonzero inputs one can store the norm separately, quantize $x/\|x\|_2$, and rescale the decoded vector. In the following, we assume that $d$ is padded to a power of $2$, thus there exists a $d\times d$ Hadamard matrix.

Our main result shows that randomized Hadamard mappings $HD$, after dithering,  provides compression guarantees that asymptotically match those achievable with truly random rotation matrices $R$. Moreover, our quantizer is unbiased, which is a desirable
property for numerous applications. In particular we show that DASHQuant has the following guarantees:

\begin{theorem}
\label{thm:vector-main}
For every dimension $d$ and every unit vector $x\in \calS^{d-1}$, \cref{alg:hsgb-unbiased} satisfies
\[\E_{D, U}[\widetilde x] = x\]
and
\begin{equation}\label{eq:main-vector}
    \E_{D,U}\|x-\widetilde x\|_2^2
    \le \frac{\pi\sqrt3/2+o(1)}{4^b},
    \qquad b\to\infty,
\end{equation}
where the $o(1)$ term vanishes as $b \to \infty$, uniformly over $d$ and $x$.
\end{theorem}

The scaling constant in this  theorem matches the bound obtained for TurboQuant in~\cite{zandieh2025turboquant} that assumed {\em fully random} rotation matrices, up to the $o(1)$ term. 
Furthermore, as noted in~\cite{ben2026note}, this also provides guarantees for a special case of the EDEN quantization method~\cite{vargaftik2022eden}, also using fully random rotation matrices. Thus, \cref{thm:vector-main} can be viewed as justifying the use a single randomized Hadamard matrix in those methods.

We also provide a quantizer for inner products. Similarly to TurboQuant, we achieve this result by additionally quantizing the residual and showing that the quantization error is uncorrelated with the direction of the vector we dot product $x$ with. Somewhat surprisingly, a classical mixed moment property of Rademacher random variables suffices to establish this fact with random rotations replaced by only a single $HD$ matrix:

\begin{theorem}\label{thm:inner-product-error} Let $x \in \calS^{d-1}$ be a unit vector and let
$y \in \R^d$. Let $\widehat{x}$ be the reconstructed vector produced by \cref{alg:inner_product_quantizer} with a base bit-width $b$.
Then
\[
    \E |\ip{y}{\widehat{x} - x}|^2
    \le
    13\left(\frac{\pi\sqrt{3}}{2}+1+o(1)\right)
    \frac{\norm{y}^2}{d \cdot 4^b},
\]
where the $o(1)$ term vanishes as $b \to \infty$, uniformly over $d$ and $x$.

Moreover, the output of \textsc{Quantize}$(x)$ from \cref{alg:inner_product_quantizer} can be stored using 
$d b + (3 + \frac{1}{2\ln 2})d + O(\log(b + \log d))$
total bits. 
\end{theorem}

\subsection{Related work}

In addition to the discussion in the introduction, in this section we elaborate on connections with some prior/concurrent works in more detail. 

 Dithering, i.e., the use of a random offset in quantization,  has a long history in quantization theory~\cite{schuchman1964dither,zamir1992universal}. In the context of random projections combined with quantization, dithering has been employed in~\cite{datar2004locality,dong2008asymmetric} and more recently in~\cite{gao2024rabitq}. However, their analysis assumed a truly random matrix $R$.  In contrast, our approach uses the randomized Hadamard transform matrix $HD$ combined with {\it non-uniform} dithering: we apply a uniform random shift in the quantile space, which corresponds to a non-uniform shift in the original coordinate space. This allows us to obtain guarantees using only a single randomized Hadamard transform.

Another approach to analyzing quantization after a single randomized Hadamard transform appears in~\cite{ben2024accelerating},  where the authors use the Bentkus--Dzindzalieta tail bound  to compare the tails of normalized Hadamard coordinates with Gaussian tails. They then partition the range of coordinates into intervals and bound the quantization error on each interval separately, leveraging the tail bound to control contributions from the outer intervals. 
However, as the authors themselves note, this approach yields only a loose bound on the expected distortion.




An alternative approach to bridging the gap between a single randomized Hadamard transform and a fully random rotation is taken by recent concurrent works \cite{zilca2026approximating,BenBasatEtAl2026}, which show that composing {\em two} randomized Hadamard transforms suffices to approximate the marginal distributions of a uniform random rotation or a Gaussian in both Kolmogorov and Wasserstein-1 distances, with an error of $\widetilde{O}(d^{-1/2})$ per coordinate. In addition, \cite{BenBasatEtAl2026} uses these guarantees to establish performance bounds for quantization schemes. While these algorithms do not require dithering, our approach uses only a single randomized Hadamard transform, achieving an asymptotic MSE bound whose lower-order term vanishes as the bit width increases, uniformly over all dimensions.

\section{Construction of DASHQuant: the MSE guarantee}

For the ease of exposition, we separate the main mean squared error (MSE) argument from the unbiasedness property. We first present a simpler baseline algorithm, \cref{alg:hsgb}, which achieves the claimed mean-squared error bound (\cref{eq:main-vector}) but is slightly biased. In the later \cref{subsec:unbiased}, we present \cref{alg:hsgb-unbiased}, which is an unbiased variant of \cref{alg:hsgb}.

Our quantizer independently and randomly flips the signs of the input vector coordinates, mixes the coordinates using a normalized Hadamard transform, and then applies the same scalar quantizer to every transformed coordinate. The scalar quantizer is a dithered version of Lloyd quantization, which was also used in TurboQuant \cite{zandieh2025turboquant} and EDEN \cite{vargaftik2022eden}.  We define the scalar quantizer now, start by introducing some basic notations.
 Let
\[
    \varphi(t)=(2\pi)^{-1/2}e^{-t^2/2}
\]
be the standard Gaussian density, and define
\[
    F(t):=\frac1A\int_{-\infty}^t \varphi(s)^{1/3}\dd s, \qquad \text{ where }A:=\int_\R \varphi(t)^{1/3}\dd t.
\]

Such $F$ is the CDF of a Gaussian random variable with mean $0$ and variance $3$ (see \cref{clm:F-distribution} for the proof). In \cref{alg:hsgb}, the scalar quantizer applies  $F$ to each coordinate, imposes a shifted uniform grid on the resulting quantile levels, and reconstructs each grid bucket at the inverse image of its midpoint, i.e. reconstructs each bucket at the corresponding middle quantile.

Specifically, we choose a uniformly random shift parameter $U\in [0, 1]$ and denote the parameter $B = 2^b$ (the target codebook size, i.e. number of possible quantized scalar values). Throughout this paper, we use $[B]$ to denote $\{0, 1, ... B - 1\}$. For $j \in [B]$, we  define grid points 
\[
    h_j:= (j+U)/B.
\]
Define  $h_0:=0$ and $h_B:=1$. Now for $j \in [B]$ define buckets 
\[
    \mathcal B_j\gets\{t\in\R: h_j\le F(t)<h_{j+1}\}.
\]
Finally, for $j \in [B]$ define reconstruction values 
\begin{equation}\label{eq:q-def}
q_j\gets F^{-1}\bigl((h_j+h_{j+1})/2\bigr).
\end{equation}
See Fig.~\ref{fig:scalar_quantizer} below for an illustration.

\begin{figure}[htpb]
    \hspace{-0.8cm}
    \begin{tikzpicture}[
        declare function={
            F(\t) = 1 / (1 + exp(-1.5*\t));
            Finv(\y) = -ln(1/\y - 1)/1.5;
            f(\t) = 3 * exp(-1.5*\t) / ((1 + exp(-1.5*\t))^2);
        }
    ]
    \begin{axis}[
        width=14cm, height=9cm,
        axis lines=left,
        xmin=-4.2, xmax=4.2,
        ymin=0, ymax=1.15,
        xlabel={$t$ (Input Coordinate)},
        ylabel={$F(t)$ (Quantile Space)},
        xtick={0}, 
        ytick=\empty,
        clip=false,
        every axis x label/.style={at={(current axis.right of origin)},anchor=north west},
        every axis y label/.style={at={(current axis.above origin)},anchor=south east},
        ylabel style={
        at={(axis description cs:0.15,0.95)},
        anchor=south,
        },
        xlabel style={
            at={(axis description cs:0.9, 0.0)},
            anchor=north
        },
    ]

    \draw[dashed, gray] (-4.2, 1) -- (4.2, 1);

    \draw[dotted, thick, gray] (0, 0) -- (0, 1.15);

    \addplot [green!60!black, thick, smooth, domain=-4:4, samples=100] {f(x)} 
        node[pos=0.72, right, text=green!60!black] {$f(t) \propto F'(t)$};

    \addplot [blue, very thick, smooth, domain=-4:4, samples=100] {F(x)} 
        node[pos=0.85, above left, text=blue] {$F(t)$};

    \node[left] at (-4.2, 0) {$h_0$};
    
    \pgfmathsetmacro{\tOne}{Finv(0.3)}
    \draw[gray, densely dotted, thick] (-4.2, 0.3) node[left, text=black] {$h_1$} -- (\tOne, 0.3) -- (\tOne, 0);
    
    \pgfmathsetmacro{\tTwo}{Finv(0.55)}
    \draw[gray, densely dotted, thick] (-4.2, 0.55) node[left, text=black] {$h_2$} -- (\tTwo, 0.55) -- (\tTwo, 0);
    
    \pgfmathsetmacro{\tThree}{Finv(0.8)}
    \draw[gray, densely dotted, thick] (-4.2, 0.8) node[left, text=black] {$h_3$} -- (\tThree, 0.8) -- (\tThree, 0);

    \node[left] at (-4.2, 1) {$h_B$};

    \pgfmathsetmacro{\qZero}{Finv(0.15)}
    \node[left, text=red, font=\footnotesize] at (-4.2, 0.15) {$m_0$};
    \draw[red, dashed] (-4.2, 0.15) -- (\qZero, 0.15) -- (\qZero, 0);
    \fill[red] (\qZero, 0.15) circle (2pt);
    \fill[red] (\qZero, 0) circle (2pt) node[below=2pt, text=red, font=\bfseries] {$q_0$};

    \pgfmathsetmacro{\qOne}{Finv(0.425)}
    \node[left, text=red, font=\footnotesize] at (-4.2, 0.425) {$m_1$};
    \draw[red, dashed] (-4.2, 0.425) -- (\qOne, 0.425) -- (\qOne, 0);
    \fill[red] (\qOne, 0.425) circle (2pt);
    \fill[red] (\qOne, 0) circle (2pt) node[below=2pt, text=red, font=\bfseries] {$q_1$};

    \pgfmathsetmacro{\qTwo}{Finv(0.675)}
    \node[left, text=red, font=\footnotesize] at (-4.2, 0.675) {$m_2$};
    \draw[red, dashed] (-4.2, 0.675) -- (\qTwo, 0.675) -- (\qTwo, 0);
    \fill[red] (\qTwo, 0.675) circle (2pt);
    \fill[red] (\qTwo, 0) circle (2pt) node[below=2pt, text=red, font=\bfseries] {$q_2$};

    \pgfmathsetmacro{\qThree}{Finv(0.9)}
    \node[left, text=red, font=\footnotesize] at (-4.2, 0.9) {$m_3$};
    \draw[red, dashed] (-4.2, 0.9) -- (\qThree, 0.9) -- (\qThree, 0);
    \fill[red] (\qThree, 0.9) circle (2pt);
    \fill[red] (\qThree, 0) circle (2pt) node[below=2pt, text=red, font=\bfseries] {$q_3$};

    \draw [decorate,decoration={brace,amplitude=6pt,mirror,raise=22pt}]
        (-4.2,0) -- (\tOne,0) node [midway,yshift=-39pt] {$\mathcal{B}_0$};
        
    \draw [decorate,decoration={brace,amplitude=6pt,mirror,raise=22pt}]
        (\tOne,0) -- (\tTwo,0) node [midway,yshift=-39pt] {$\mathcal{B}_1$};
        
    \draw [decorate,decoration={brace,amplitude=6pt,mirror,raise=22pt}]
        (\tTwo,0) -- (\tThree,0) node [midway,yshift=-39pt] {$\mathcal{B}_2$};
        
    \draw [decorate,decoration={brace,amplitude=6pt,mirror,raise=22pt}]
        (\tThree,0) -- (4.2,0) node [midway,yshift=-39pt] {$\mathcal{B}_3$};

    \draw [decorate,decoration={brace,amplitude=6pt,raise=22pt}]
        (-4.2,0.3) -- (-4.2,0.55) node [midway,xshift=-40pt] {$\frac{1}{B}$};

    \draw [decorate,decoration={brace,amplitude=6pt,raise=22pt}]
        (-4.2,0.55) -- (-4.2,0.8) node [midway,xshift=-40pt] {$\frac{1}{B}$};

    \draw [decorate,decoration={brace,amplitude=6pt,raise=22pt}]
        (-4.2,0) -- (-4.2,0.3) node [midway,xshift=-40pt] {$\frac{1 + U}{B}$};
        
    \draw [decorate,decoration={brace,amplitude=6pt,raise=22pt}]
        (-4.2,0.8) -- (-4.2,1) node [midway,xshift=-40pt] {$\frac{1 - U}{B}$};
        
    \end{axis}
    \end{tikzpicture}
    \caption{Illustration of the scalar quantizer mapping with $B=4$. The function $F(t)$ (blue) and its scaled underlying density $f(t)$ (green) are shown. A uniform grid shifted by $U$ defines boundaries $h_j$ on the y-axis, which project back to the boundaries of the quantization buckets $\mathcal{B}_j$. The reconstruction values $q_j$ are defined as the inverse image of the bucket midpoints $m_j = (h_j+h_{j+1})/2$.}
    \label{fig:scalar_quantizer}
\end{figure}

Our vector quantization algorithm \textsc{VectorQuant}, presented in~\cref{alg:hsgb}, applies the scalar dithered quantization scheme described above to every coordinate of the input vector $x\in \mathcal{S}^{d-1}$ after applying the $HD$ matrix.

For the purposes of the analysis it is convenient to think of the quantizer as returning the corresponding middle quantile as opposed to its index. To this effect, it is convenient to introduce
\begin{definition}[$\quant$ function]\label{def:quant}
 For $z\in \mathbb R$ define $\quant(z)$ to return the middle quantile $q_j$ of the bucket $\mathcal{B}_j$ that contains $z$, i.e. $\quant(z) := q_j$ for $z \in \mathcal{B}_j$.
\end{definition}
With Definition~\ref{def:quant} in place the task of bounding quantization error in all coordinates becomes, letting $z=HDx$, the task of bounding
\[
\E_{\eps, U} \sum_{i\in [d]}(z_i-\quant(z_i))^2=d\cdot \E_{\eps, U} (z_1-\quant(z_1))^2,
\]
where the last equality uses the fact that the marginal distributions of coordinates of $z$ are the same, i.e. they are all distributed as the Rademacher sum $\sum_{j\in [d]} \eps_j x_j$.

The details of the algorithms are depicted in \cref{alg:hsgb}. Recall $[B]$ always denotes $\{0, 1, \ldots, B-1\}$.

\begin{algorithm}
\small
\caption{Dithered scalar Gaussian quantization}
\label{alg:hsgb}
\begin{algorithmic}[1]
\State \textbf{input:} dimension $d$ and bit-width $b$
\State Let $H\in\R^{d\times d}$ be a normalized Hadamard matrix  and let $B\gets 2^b$.
\State Sample independent Rademacher random variables $\eps_1,\ldots,\eps_d$ and set $D\gets\diag(\eps_1,\ldots,\eps_d)$.
\State Sample an offset $U\sim\Unif[0,1)$.

\Statex \hrulefill

\Procedure{ConstructScalarCodebook}{$B, U$}
    \State Define grid points $h_j\gets (j+U)/B$ for $j \in [B]$.\Comment{Uniformly shifted grid in $[0, 1]$}
\Statex \hspace{2.6em}Define  $h_0\gets0$ and $h_B\gets1$. 
\State Define buckets $\mathcal B_j\gets\{t\in\R: h_j\le F(t)<h_{j+1}\}$  for $j \in [B]$.
\State Define reconstruction values $q_j\gets F^{-1}\bigl((h_j+h_{j+1})/2\bigr)$ for $j \in [B]$.
\EndProcedure

\Statex \hrulefill

\Procedure{VectorQuant}{$x\in \calS^{d-1}$}
    \State $y\gets HDx$
    \For{$i\in[d]$}
    \State $z_i\gets\sqrt d\,y_i$ 
    \State $\idx_i\gets $ index $j\in [B]$ such that $z_i\in \mathcal{B}_j$ (i.e., the index of $z_i$'s bucket)
    \EndFor
    \State \Return $\idx = (\idx_1,\ldots,\idx_d)\in [B]^d$ 
\EndProcedure

\Statex \hrulefill

\Procedure{VectorDeQuant}{$\idx$}
    \For{$i\in[d]$}
    \State $\widetilde y_i\gets (1/\sqrt d) \cdot q_{\idx_i}$  \Comment{$q_{\idx_i}$ is the middle quantile of bucket $\mathcal B_{\idx_i}$, see~\cref{eq:q-def}}
    
    \EndFor
    \State $\widetilde x\gets D H^\top\widetilde y$
    \State \Return $\widetilde x$
\EndProcedure
\end{algorithmic}
\end{algorithm}

The main guarantee of the \Cref{alg:hsgb} is the following \cref{thm:vector-main-biased}.

\begin{theorem}
\label{thm:vector-main-biased}
For every dimension $d$ and every unit vector $x\in \calS^{d-1}$, \cref{alg:hsgb} satisfies
\begin{equation*}
    \E_{D,U}\|x-\widetilde x\|_2^2
    \le \frac{\pi\sqrt3/2+o(1)}{4^b},
    \qquad b\to\infty,
\end{equation*}
where the $o(1)$ term vanishes as $b \to \infty$, uniformly over $d$ and $x$.
\end{theorem}

To achieve the additional unbiasedness property, \cref{alg:hsgb-unbiased} makes two changes
to the scalar quantizer. First, instead of pinning the grid endpoints $h_0, h_B$ at $0$
and $1$, it uses a shifted grid directly on the quantile levels $F(t)$, so
that every input $t$ has a uniformly random position within a bucket. Second,
instead of reconstructing each bucket at the corresponding middle quantile by applying $F^{-1}$ to the midpoint, it
uses a modified reconstruction map $G$, carefully chosen so that averaging $G$ 
over
one bucket recovers the middle quantile value. This makes the scalar
reconstruction unbiased, and hence the vector estimator unbiased, while $G$
remains close enough to $F^{-1}$ on the central region to preserve the same
 MSE bound. For details, see \cref{alg:hsgb-unbiased} in \cref{subsec:unbiased}.


\section{Proof sketch of \cref{thm:vector-main-biased}}
In this section, we outline our proof ideas for  \cref{thm:vector-main-biased}.
We can use orthogonality of $HD$ to decompose the error coordinate-wise:
\begin{gather*}
    \|x-\widetilde x\|_2^2
    =
    \|x- DH^{\top} \widetilde y\|_2^2
    =
    \|HDx-\widetilde y\|_2^2
    = \\
    \sum_{i=1}^d \bigl (y_i - \frac{1}{\sqrt{d}}\quant(\sqrt{d} y_i) \bigr)^2
    =
    \frac{1}{d}\sum_{i=1}^d \bigl (z_i - \quant(z_i) \bigr)^2.
\end{gather*}
Therefore, to control the  total MSE: $\E_{D,U}\|x-\widetilde x\|_2^2$,  it is enough to bound the variance of the scalar quantizer $\quant(\cdot)$ under the specific input distibution of $z_i$. For every coordinate $i$, this distribution is given by
\[
    z_i = \sqrt{d}(HDx)_i = \sqrt{d}\sum_{j} H_{ij}D_{jj}x_j = \sum_{j} \eps_j x_j
\]
where $\eps_j$ are i.i.d. Rademacher random variables taking values $\pm 1$ with probability $1/2$. In contrast to using a truly random rotation or an i.i.d. Gaussian matrix, where $z_i$ distributed as $\mathcal{N}(0, \|x\|^2_2)$, our case is more delicate, as distribution of $\sum_{j} \eps_j x_j$ does not depend only on $\|x\|_2$,  but also on individual coordinates. It means that our oblivious scalar  quantizer $\quant$ should provide guarantees for \textit{any} input distribution of the form $z_i = \sum_{j} \eps_j x_j$, where $\|x\|_2 = 1$. We therefore drop index $i$ in what follows. Recall that $B = 2^b$ is the number of buckets. Our main guarantee for $\quant$ is the following.

\begin{lemma}[Restated \cref{lem:scalar-main}]
\label{lem:quant-main}
For every dimension $d$, every unit vector $x\in \calS^{d-1}$ define $z = \sum_{j} \eps_j x_j$ where i.i.d. $\eps_j = \pm 1$ with probability $1/2$. Then
\[
    \E_{\eps, U} |z - \quant(z)|^2
    \le \frac{\pi\sqrt3/2+o(1)}{B^2},
    \qquad B\to\infty,
\]
where the $o(1)$ term is uniform in $d$ and $x$.
\end{lemma}
As we described above, \Cref{thm:vector-main} follows directly from \Cref{lem:quant-main}. The rest of this section is dedicated to \Cref{lem:quant-main}.

\Cref{lem:quant-main} is a bit surprising - the codebook of $\quant$ is built from the optimal quantization nodes for $\mathcal{N}(0, 1)$, but we prove that the same error rate is achieved for a different distribution: the Rademacher sums $z = \sum_{j} \eps_j x_j$. This distribution can be quite different from $\mathcal{N}(0, \|x\|^2_2)$ for worst-case $x$. For example if $x = (1, 0, \dots 0)$ then $z$ only takes values $\pm 1$. However, it turns out that to prove \cref{lem:quant-main}, we only need the fact that $z$ is 1-subgaussian. This is captured by the following claim.
\begin{lemma}
    \label{lem:subgauss}
    For every dimension $d$, every unit vector $x\in \calS^{d-1}$ define $z = \sum_{j} \eps_j x_j$ where i.i.d. $\eps_j = \pm 1$ with probability $1/2$. Then $z$ is a $1$-subgaussian random variable, i.e.
    \[
        \E e^{\lambda z}\le e^{\lambda^2/2}\qquad\text{for all }\lambda\in\R.
    \]
\end{lemma}

Note that subgaussianity is a powerful property that in some cases allows us to simply substitute the distribution with the standard Gaussian. More formally, we will use the following (see \Cref{lem:rademacher-inputs}):
\begin{claim}
    \label{claim:subgauss}
    If $\alpha > 0$, $z$ is $1$-subgaussian and $\xi \sim \mathcal{N}(0, 1)$, then $\E e^{\alpha z^2} \le \E e^{\alpha \xi^2}$.
\end{claim}

Unfortunately, there is no straightforward way to apply this property to bound the expectation of the quantization error because the error function $\min_{q_j} |t - q_j|^2$ is highly non-monotone -- it could be that subgaussian distribution tends to concentrate around highs of the error function.

We now move on to outline our proof for \Cref{lem:quant-main}. In fact, we will prove a bit more general version of \Cref{lem:quant-main} with any $1$-subgaussian input.\footnote{Note that $\mathcal{N}(0, 1)$ appearing from a truly random rotation is also $1$-subgaussian.} From now on, we will write $\E_z$ instead of $\E_\eps$.

We separate the analysis of $\E_{z, U} |z - \quant(z)|^2$ according to whether $z$ lies in a ``central interval'' $[-M, M]$, i.e. we separately consider a \textit{tail} event $|z| > M$ and a \textit{central} event $|z| \leq M$, for some carefully chosen threshold $M = C\sqrt{\log B}, C \in (4, 6)$. Such a narrow choice of values of $C$ is due to the following tension: we need to guarantee that a $O(\frac{1}{B})$ shift in quantiles of $F$ (see \Cref{alg:hsgb}) results in a $o(1)$ shift in centroids in the central interval; this forces $M$ not to be too large. But with the same choice of $M$ we need the tail event to be unlikely enough that we can ignore the centroids construction outside of the central interval completely; this forces $M$ to be not too small.

\subsection{Central event}
In the \textit{central} event, we argue that buckets are narrow enough to treat standard Gaussian density on them as uniform. Then, to recover optimal error rate for our arbitrary $1$-subgaussian distribution, we induce uniformity by \textit{dithering}, i.e. applying random shift on the quantization grid. More specifically, as we show in \Cref{lem:central-linearization}, for our choice of $M$, any $|t| < M$ and quantile shift $\frac{U}{B} \sim \Unif[0; 1/B)$ we have the following
\begin{itemize}
    \item $T := \Bigl[F(t) - F(\quant(t))\Bigr]$ is uniform on $[-\frac{1}{2B}; \frac{1}{2B}]$.
    \item $\frac{F'(s)}{F'(t)} = 1 + o_B(1)$ for all $s \in [t; \quant(t)]$ where $o_B(1)$ term is uniform in $s$ and $U$.
\end{itemize}
Then we can use first-order approximation and write
\[
    \quant(t) - t = F^{-1}\left(F(t)+\frac{T}{B}\right) - t = (1+o_B(1))\frac{T/B}{F'(t)}.
\]
And, integrating over $T$, we bound
\[
    \E_U |t - \quant(t)|^2
    \le
    (1+o_B(1))\frac{1}{12B^2}\frac{1}{F'(t)^2}.
\]
Note that the error $\frac{T/B}{F'(t)}$, $T \sim \Unif[-\frac{1}{2}, \frac{1}{2}]$ matches the typical error in a bucket of width $\frac{1}{B F'(t)}$ with a centroid placed in the middle. Since $F'(s)$ is almost constant over the bucket, this calculation differs from an actual error only by lower order terms.

To get the total error bound we first take expectation over $U$
\[
    \E_{z,U} |z - \quant(z)|^2
    \le
    (1+o_B(1))\frac{1}{12B^2}\E_z\frac{1}{F'(z)^2}
\]
and then, by \Cref{claim:subgauss}, since $F'(x)$ is a Gaussian density and $z$ is $1$-subgaussian,
\[
    \E_z\frac{1}{F'(z)^2} \leq \E_\xi\frac{1}{F'(\xi)^2}, \quad\text{where } \xi \sim \mathcal{N}(0, 1).
\]
Finally,
\[
    \E_{z,U} |z - \quant(z)|^2
    \le
    (1+o_B(1))\frac{1}{12B^2}\E_\xi\frac{1}{F'(\xi)^2} = \frac{\pi\sqrt3/2+o(1)}{B^2}
\]

\subsection{Tail event}
In the \textit{tail} event, we give a simple bound
\[
    |z - \quant(z)|^2 \leq 2|z|^2 + 2|\quant(z)|^2
\]
and argue that both terms on the RHS are bounded in expectation by $O(\poly\log B \cdot e^{-M^2/2}) = o(\frac{1}{B^2})$ using subgaussianity of $z$ and our choice of $M$. A little more work is required to carefully bound $|\quant(z)|$ since the centroid construction is randomized. For more details, see \Cref{lem:global-envelope}.

\section{Construction of DASHQuant: the inner product guarantee}\label{sec:inner-product}
In this section we present our inner product quantizer and prove Theorem~\ref{thm:inner-product-error}. The quantizer is presented as Algorithm~\ref{alg:inner_product_quantizer} below. The main idea is to add a second stage to the quantization process that decorrelates the residual $r=x-\widetilde x$ with any fixed vector $y\in \R^d$. As a result, we obtain the approximation error for $(\langle y, x\rangle-\langle y, \widetilde x\rangle)^2$ that scales as $O(\|y\|_2^2 /(d B^2))$, as required.  This decorrelation is achieved by a single application of a fresh $HD$ matrix to the residual followed by a simple quantization procedure that uses a constant number of bits per coordinate in expectation. We now describe this basic scalar quantizer. 

Concretely, let  $r=x-\widetilde x$ be the residual vector, where $\widetilde x$ is the result of quantization using our main \textsc{VectorQuant} procedure followed by projecting onto the unit ball -- see line~\ref{line:project}. The projection operator $\Pi$ is defined as $\Pi(v) =v$ if $\|v\|_2\leq 1$ and $\Pi(v) = v/\|v\|_2$ otherwise; since the input $x$ is a unit norm vector, this only improves the quality of reconstruction, and helps us ensure a deterministic upper bound on the norm of the residual.

The residual $r$ is then quantized by first quantizing the {\em residual
scale} (line~\ref{line:res-scale}), defined  by  $s =\frac{\norm{r}}{\sqrt d} \in [0, \frac{2}{\sqrt{d}}]$, and then quantizing the individual elements (lines~\ref{line:1}-\ref{line:2}). This scale $s$ can be interpreted as  a reference magnitude for a typical coordinate after the transform.
 We first quantize the scale $s$ using a basic scalar quantizer $\textsc{ScalarQuant}(s)$ (Algorithm~\ref{alg:scalar_quantization} below) and then quantize the residual's coordinates based on the quantized scale (Algorithm~\ref{alg:residual_quantizer} below).

\begin{algorithm}
\small
\caption{Scalar quantization}
\label{alg:scalar_quantization}
\begin{algorithmic}[1]
\State \textbf{input:} scalar $s\in [0, \frac{2}{\sqrt{d}}]$, dimension $d$, scale parameter $B\ge2$
\State \textbf{public parameter:} minimum encoded scale $\tau_B\gets 1/(dB)$.

\Statex \hrulefill

\Procedure{ScalarQuant}{$s$}
    \If{$s<\tau_B$}
        \State \Return $\idx_\sigma\gets0$ \Comment{Truncate scales smaller than $\tau_B=1/(d B)$ down to zero}
    \EndIf
    \State $k\gets \lceil \log_2(s/\tau_B)\rceil$ \Comment{A factor two approximation for scales larger than $\tau_B$}
    \State $\idx_\sigma\gets k+1$
    \State \Return $\idx_\sigma$ 
\EndProcedure

\Statex \hrulefill

\Procedure{ScalarDeQuant}{$\idx_\sigma$}
    \If{$\idx_\sigma=0$}
        \State \Return $\sigma\gets0$
    \EndIf
    \State \Return $\sigma\gets \tau_B2^{\idx_\sigma-1}$
\EndProcedure
\end{algorithmic}
\end{algorithm}

\begin{algorithm}
\small
\caption{Residual Quantization}
\label{alg:residual_quantizer}
\begin{algorithmic}[1]
\State \textbf{input:} dimension $d$, scale parameter $B\ge2$
\State Let $H\in\R^{d\times d}$ be a normalized Hadamard matrix and $D$ be a diagonal matrix whose entries are i.i.d. Rademacher random variables.

\Statex \hrulefill

\Procedure{ResidualQuant}{$r\in \R^d$}
    \State $\idx_\sigma  \gets \Call{ScalarQuant}{\norm{r}/\sqrt d}$

    \If{$\idx_\sigma  = 0$}
                \State \Return $\idx_\sigma$, $\ell=(0,\ldots,0)$, and $\lambda =(0,\ldots,0)$
    \EndIf
    \State $\sigma\gets\Call{ScalarDeQuant}{\idx_\sigma}$
    \State $v \gets HDr$
    \For{$i\in [d]$}
        \State $\ell_i \gets \min\{\ell\in\{0,1,2,\ldots\}: |v_i|\le \sigma \cdot 2^\ell\}$ \Comment{ $\ell_i$ can be stored, e.g., as unary word $1^{\ell_i}0$}
        \State $R_i \gets \sigma\cdot 2^{\ell_i}$
        \State Sample a sign bit $\lambda_i\in\{-1,1\}$ such that $\Prb(\lambda_i=1\mid D,r) = \frac{1+(v_i/R_i)}{2}$
    \EndFor
    \State \Return $\idx_\sigma$, $\ell=(\ell_1,\ldots,\ell_d)$, and $\lambda=(\lambda_1,\ldots,\lambda_d)$
\EndProcedure

\Statex \hrulefill

\Procedure{ResidualDeQuant}{$\idx_\sigma, \ell, \lambda$}
    \State $\sigma\gets\Call{ScalarDeQuant}{\idx_\sigma}$
    \For{$i\in [d]$}
        \State $R_i \gets \sigma \cdot 2^{\ell_i}$
        \State $q_i \gets R_i \lambda_i$
    \EndFor
    \State $\widehat r \gets D H^\top q$
    \State \Return $\widehat r$
\EndProcedure
\end{algorithmic}
\end{algorithm}

The goal of \cref{alg:scalar_quantization}  is to use a bounded number of bits to quantize the scale $s$ of the residual vector to a dyadic estimate $\sigma$. Once we have this, in the following step (Algorithm~\ref{alg:residual_quantizer}), we can then quantize each coordinate of the residual vector with an appropriate level of precision.  Concretely, in  \cref{alg:scalar_quantization}, if the residual scale exceeds a minimum cutoff, i.e., $s\ge\tau_B$, then the quantized scale is
$\sigma = \tau_B2^{\idx_\sigma-1}$ satisfying $s\le\sigma<2s$.  On the other hand if $s<\tau_B$, then the quantized scale is set to 
$\sigma=0$ indicating that the residual is negligibly small (thus the residual quantization stage can be skipped). The quantized scales are indexed by $\idx_\sigma$,  which can be encoded using
$\lceil\log_2(\lceil \log_2 (\frac{2}{\sqrt{d}}/\tau_B) \rceil +2)\rceil = O(\log\log(dB))$
bits when $s\in [0, \frac{2}{\sqrt{d}}]$.

In English, for each coordinate $i \in [d]$, \cref{alg:residual_quantizer} records the level $\ell_i$, which reflects the absolute value of the coordinate, and separately records its sign. For simplicity the level $\ell_i$ is recorded in unary as the word $1^{\ell_i}0$ (of course, this can be improved somewhat by a binary encoding, but the asymptitic complexity will not change). Finally, the quantizer samples one randomized sign bit
$\lambda_i\in\{-1,1\}$ with
$\Prb(\lambda_i=1\mid D,r)
  =
  \frac{1+(v_i/R_i)}{2}$
to ensure that our estimator is unbiased.  This probability is valid because $|v_i|\le \sigma 2^{\ell_i} = R_i$.  The intuition behind this construction is that   the random value $R_i\lambda_i$ is an unbiased estimate of $v_i$: $\mathbb{E}[R_i \lambda_i]=(R_i\cdot 1)\cdot \frac{1+(v_i/R_i)}{2}+(R_i\cdot (-1)) \frac{1-(v_i/R_i)}{2}=v_i$.   And since $R_i$ is chosen to be the smallest admissible dyadic radius, the variance  is also controlled, allowing the coordinate to be encoded using only the level $\ell_i$ and one sign bit $\lambda_i$.

Finally, our two-stage inner product quantizer is presented below as Algorithm~\ref{alg:inner_product_quantizer}. As summarized at the beginning of this section, it combines the MSE quantizer with a second stage that applies a fresh $HD$ matrix to the residual
 $r = x - \widetilde x$ and then quantizes the transformed residual. This construction makes the residual-quantization error sufficiently uncorrelated with any fixed query vector, yielding an inner-product error that decreases by a factor of $d$. Its analysis, namely the proof of Theorem~\ref{thm:inner-product-error}, is presented in the appendix (\cref{app:inner-product}).
\begin{algorithm}[h]
\small
\caption{Two-Stage Inner Product Quantizer}
\label{alg:inner_product_quantizer}
\begin{algorithmic}[1]
\State \textbf{input:} dimension $d$, base bit-width $b$
\State Let $H\in\R^{d\times d}$ be a normalized Hadamard matrix, let $B\gets 2^b$, and let $\Pi: \R^d \to \R^d$ be a projection to the unit ball.
\State Sample independent Rademacher random variables $\eps_1, \ldots, \eps_d$ and set $D_{\mathrm{base}}\gets\diag(\eps_1,\ldots,\eps_d)$.
\State Sample independent Rademacher random variables $\eta_1, \ldots, \eta_d$ and set $D_{\mathrm{res}}\gets\diag(\eta_1,\ldots,\eta_d)$.
\State Sample an offset $U\sim\Unif[0,1)$.
\State \Call{ConstructScalarCodebook}{$B, U$} \Comment{Defined in \cref{alg:hsgb}. This constructs $\mathcal{B}_j$ and $q_j$ for $j \in [B]$}

\Statex \hrulefill

\Procedure{Quantize}{$x\in \calS^{d-1}$}
    \State \Comment{\textbf{Stage 1: Base Dithered Quantization}}
    \State $y\gets H D_{\mathrm{base}} x$
    \For{$i\in[d]$}
        \State $z_i\gets\sqrt d\,y_i$
        \State $\idx_i\gets $ index $j\in [B]$ such that $z_i\in \mathcal{B}_j$ \Comment{$\mathcal{B}_j$'s computed by \Call{ConstructScalarCodebook}{$B, U$}}
        \State $\widetilde y_i\gets (1/\sqrt d) \cdot q_{\idx_i}$ \Comment{Local dequantization for residual}
    \EndFor
    \State $\idx \gets (\idx_1,\ldots,\idx_d)\in [B]^d$
    \State $\widetilde x\gets \Pi D_{\mathrm{base}} H^\top\widetilde y$ \label{line:project} \Comment{Base reconstructed vector, projected}

    \State \Comment{\textbf{Stage 2: Residual Quantization}, repeated from \cref{alg:residual_quantizer}}
    \State $r \gets x - \widetilde x$ \Comment{Compute the residual vector}
        \State $\idx_\sigma \gets \Call{ScalarQuant}{\norm{r}/\sqrt d}$ \label{line:res-scale}
    \If{$\idx_\sigma = 0$}
        \State \Return $\idx$, $\idx_\sigma$, $\ell=(0,\ldots,0)$, $\lambda=(0,\ldots,0)$
    \EndIf
    \State $\sigma\gets\Call{ScalarDeQuant}{\idx_\sigma}$
    \State $v \gets H D_{\mathrm{res}} r$
    \For{$i\in [d]$}\label{line:1}
        \State $\ell_i \gets \min\{\ell\in\{0,1,2,\ldots\}: |v_i|\le \sigma 2^\ell\}$ \Comment{ $\ell_i$ can be stored, e.g., as unary word $1^{\ell_i}0$}
        \State $R_i \gets \sigma \cdot 2^{\ell_i}$
        \State Sample sign bit $\lambda_i\in\{-1,1\}$ such that $\Prb(\lambda_i=1\mid D_{\mathrm{res}},r) = \frac{1+v_i/R_i}{2}$
    \EndFor\label{line:2}
    \State \Return $\idx$, $\idx_\sigma$, $\ell=(\ell_1,\ldots,\ell_d)$, and $\lambda=(\lambda_1,\ldots,\lambda_d)$
\EndProcedure

\Statex \hrulefill

\Procedure{DeQuantize}{$\idx, \idx_\sigma, \ell, \lambda$}
    \State \Comment{\textbf{Stage 1: Base Reconstruction}}
    \For{$i\in[d]$}
        \State $\widetilde y_i\gets (1/\sqrt d) \cdot q_{\idx_i}$
    \EndFor
    \State $\widetilde x\gets \Pi D_{\mathrm{base}} H^\top\widetilde y$

    \State \Comment{\textbf{Stage 2: Residual Reconstruction}, repeated from \cref{alg:residual_quantizer}}
    \State $\sigma\gets\Call{ScalarDeQuant}{\idx_\sigma}$
    \If{$\sigma = 0$}
        \State $\widehat r \gets 0$
    \Else
        \For{$i\in [d]$}
            \State $R_i \gets \sigma 2^{\ell_i}$
            \State $q^{\mathrm{res}}_i \gets R_i \lambda_i$
        \EndFor
        \State $\widehat r \gets D_{\mathrm{res}} H^\top q^{\mathrm{res}}$
    \EndIf

    \State \Return $\widehat x \gets \widetilde x + \widehat r$ \Comment{Final reconstructed vector}
\EndProcedure
\end{algorithmic}
\end{algorithm}

\section{Experiments}
\label{sec:experiments}

We compare the performance of DASHQuant (\cref{alg:hsgb}) with
TurboQuant~\cite{zandieh2025turboquant}, which uses a dense random rotation.
For the baseline, we use the public, third-party
implementation at
\url{https://github.com/0xSero/turboquant}. 
We implement DASHQuant in the same
codebase by replacing the dense rotation and Lloyd-Max codebook
with the randomized Hadamard transform and dithered codebook from
\cref{alg:hsgb}, while keeping the remaining data handling and measurement code
common.  Both methods quantize the same power-of-two-padded database vectors, and reconstructed vectors are cropped to
the source dimension.

\paragraph{Setup.}
We evaluate the  normalized mean-squared
error (NMSE) at seven bit widths $b\in\{2,\ldots,8\}$ on 100,000 database vectors
from GloVe \cite{pennington-etal-2014-glove} ($d=200$, padded to $256$) and two DBpedia OpenAI3 embedding sets \cite{thakur2021beir}
($d=1{,}536$ and $3{,}072$, padded to $2{,}048$ and $4{,}096$). 
The NMSE is measured on the original, unpadded coordinates:
\[
  \operatorname{NMSE}
  = \E\!\left[\frac{\lVert x-\widehat{x}\rVert_2^2}
                    {\lVert x\rVert_2^2}\right].
\]
The results are averaged over 50 quantizer seeds. We also measure the retrieval quality by $1@k$: the details and results are included in \cref{app:experiment}. All experiments were conducted on a single NVIDIA H200 GPU.

\begin{figure}[H]
  \centering
  \resizebox{\textwidth}{!}{%
  \begin{tikzpicture}
    \begin{axis}[
      at={(0cm,0cm)}, anchor=south west,
      width=5.2cm, height=3.7cm,
      title={GloVe-200},
      xlabel={Bits per coordinate}, ylabel={NMSE},
      xmin=1.8, xmax=8.2, xtick={2,3,4,5,6,7,8},
      ymin=2e-5, ymax=2.2e-1, ymode=log,
      ymajorgrids, yminorgrids,
      major grid style={gray!35}, minor grid style={gray!15},
      tick label style={font=\scriptsize},
      label style={font=\scriptsize}, title style={font=\small},
    ]
      \addplot[blue!70!black, mark=*, line width=0.9pt]
        coordinates {(2,9.424e-2) (3,2.706e-2) (4,7.385e-3)
                     (5,1.942e-3) (6,5.350e-4) (7,1.906e-4)
                     (8,7.458e-5)};
      \addplot[orange!85!black, dashed, mark=square*, line width=0.9pt]
        coordinates {(2,1.676e-1) (3,3.867e-2) (4,9.016e-3)
                     (5,2.155e-3) (6,5.253e-4) (7,1.296e-4)
                     (8,3.217e-5)};
    \end{axis}
    \begin{axis}[
      at={(5.2cm,0cm)}, anchor=south west,
      width=5.2cm, height=3.7cm,
      title={DBpedia-1536},
      xlabel={Bits per coordinate},
      xmin=1.8, xmax=8.2, xtick={2,3,4,5,6,7,8},
      ymin=2e-5, ymax=2.2e-1, ymode=log, yticklabels=\empty,
      ymajorgrids, yminorgrids,
      major grid style={gray!35}, minor grid style={gray!15},
      tick label style={font=\scriptsize},
      label style={font=\scriptsize}, title style={font=\small},
      legend style={at={(0.5,1.30)}, anchor=south, legend columns=2,
                    draw=none, font=\scriptsize, /tikz/every even column/.append style={column sep=5pt}},
    ]
      \addplot[blue!70!black, mark=*, line width=0.9pt]
        coordinates {(2,9.150e-2) (3,2.619e-2) (4,7.143e-3)
                     (5,1.879e-3) (6,5.176e-4) (7,1.848e-4)
                     (8,7.263e-5)};
      \addlegendentry{TurboQuant}
      \addplot[orange!85!black, dashed, mark=square*, line width=0.9pt]
        coordinates {(2,1.641e-1) (3,3.733e-2) (4,8.701e-3)
                     (5,2.086e-3) (6,5.096e-4) (7,1.259e-4)
                     (8,3.127e-5)};
      \addlegendentry{DASHQuant}
    \end{axis}
    \begin{axis}[
      at={(10.4cm,0cm)}, anchor=south west,
      width=5.2cm, height=3.7cm,
      title={DBpedia-3072},
      xlabel={Bits per coordinate},
      xmin=1.8, xmax=8.2, xtick={2,3,4,5,6,7,8},
      ymin=2e-5, ymax=2.2e-1, ymode=log, yticklabels=\empty,
      ymajorgrids, yminorgrids,
      major grid style={gray!35}, minor grid style={gray!15},
      tick label style={font=\scriptsize},
      label style={font=\scriptsize}, title style={font=\small},
    ]
      \addplot[blue!70!black, mark=*, line width=0.9pt]
        coordinates {(2,9.153e-2) (3,2.620e-2) (4,7.145e-3)
                     (5,1.880e-3) (6,5.178e-4) (7,1.849e-4)
                     (8,7.267e-5)};
      \addplot[orange!85!black, dashed, mark=square*, line width=0.9pt]
        coordinates {(2,1.654e-1) (3,3.742e-2) (4,8.710e-3)
                     (5,2.088e-3) (6,5.101e-4) (7,1.260e-4)
                     (8,3.130e-5)};
    \end{axis}
  \end{tikzpicture}%
  }
  \caption{NMSE as a function of bit width in logarithmic scale (lower is better).}
  \label{fig:nmse-by-bit}
\end{figure}

\paragraph{Results.}

As shown in \cref{fig:nmse-by-bit}, TurboQuant has lower reconstruction error
at low bit widths, but the gap closes and reverses as $b$ increases: the methods have essentially identical NMSE  at $b = 6$, with DASHQuant winning by a tiny marginal, after
  which DASHQuant has lower reconstruction error.
 We also time encoding plus decoding all
100,000 database vectors.  These timings exclude quantizer setup, data loading
and transfer,  and metric computation; search operates
on float32  for both methods.  \cref{fig:timing-speedup} visualizes the runtime comparison at $b=3$. The comparison depends on dimensions:
DASHQuant is slower on the relatively low-dimensional dataset GloVe, but faster on DBpedia-$1{,}536$  and $3{,}072$.   Exact values for both NMSE and the timing are reported in \cref{tab:nmse-timing} in \cref{app:experiment}.

\begin{wrapfigure}{l}{0.40\textwidth}
  \centering
  \begin{minipage}[]{\linewidth}
    \vspace{-3mm}
    \resizebox{\linewidth}{!}{%
    \begin{tikzpicture}
      \begin{axis}[
        width=5.0cm, height=4.1cm,
        ylabel={TQ/DQ time},
        symbolic x coords={GloVe,DB1536,DB3072},
        xtick={GloVe,DB1536,DB3072},
        xticklabels={GloVe,DB-1536,DB-3072},
        ymin=0, ymax=5.8, ytick={0,1,2,3,4,5},
        yticklabel={\pgfmathprintnumber{\tick}$\times$},
        enlarge x limits=0.22,
        ymajorgrids, major grid style={gray!35},
        tick label style={font=\tiny},
        label style={font=\scriptsize},
        nodes near coords={\pgfmathprintnumber[fixed,precision=2]{\pgfplotspointmeta}$\times$},
        every node near coord/.append style={font=\tiny, yshift=1pt},
        bar width=12pt,
      ]
        \addplot[ybar, bar shift=0pt, fill=blue!70!black,
                 draw=blue!70!black]
          coordinates {(GloVe,0.78)};
        \addplot[ybar, bar shift=0pt, fill=orange!85!black,
                 draw=orange!85!black]
          coordinates {(DB1536,2.55) (DB3072,4.60)};
        \addplot[gray!70!black, densely dashed, line width=0.7pt,
                 forget plot, nodes near coords={}]
          coordinates {(GloVe,1) (DB3072,1)};
      \end{axis}
    \end{tikzpicture}%
    }
  \end{minipage}
  \caption{Encode--decode speedup at $b=3$ (TQ/DQ; values above
  $1\times$ favor DASHQuant).}
  \label{fig:timing-speedup}
  \vspace{-4mm}
\end{wrapfigure}

\section*{Limitations}
This work studies a specific class of quantizers based on randomized Hadamard transform. Our theory characterizes worst-case MSE and inner-product error, while our experiments evaluate MSE, nearest-neighbor retrieval, and runtime on GloVe and DBpedia. We have not evaluated end-to-end LLM workloads such as KV-cache quantization or system-level latency and memory savings. We also do not optimize all finite-bit constants or compare against a broader range of structured transforms and scalar quantizers. These are important directions for future work.

\clearpage

\section*{Acknowledgments}
YF and PI were supported in part by the NSF TRIPODS program (award DMS-2022448) and Wistron Corporation.
Boris Prokhorov was supported by the Swiss National Science Foundation grant 10004935 - Sublinear Algorithms for Non-linear Structures (Project MINT).

\section*{AI Acknowledgments}
The authors used large language models to assist with writing, editing, and formatting portions of this manuscript. The authors take full responsibility for the final content, including all technical claims and proofs.


\bibliographystyle{alpha}
\bibliography{references.bib}

\newpage
\appendix

\section{Full proof of Theorem~\ref{thm:vector-main}}
\subsection{Scalar Estimate}\label{subsec:scalar-estimate}
In this section, we prove \cref{lem:scalar-main}, the scalar estimate that controls the coordinate-wise mean squared error. In the next section, we then use it to bound the error of vector quantization.

For any unit vector $a\in\R^d$, we define a random variable
\[
    X_a:=\sum_{i=1}^d a_i\eps_i,
\]
where $\eps_1, \ldots, \eps_d$ are independent Rademacher random variables. We abbreviate $\eps = (\eps_1, \ldots, \eps_d)$.  For the scalar quantizer appearing in \Cref{alg:hsgb}, we write
\[
    \quant(t):=q_j \quad\text{when }t \in \mathcal B_j,
\]
i.e. $\quant(\cdot)$ maps any scalar $t \in \R$ to the corresponding reconstruction value, which is the preimage of the midpoint of the bucket in which $F(t)$ lands. We recall that $\quant$ depends on $U, B$ in the following ways: $U$ is a random offset that determines the shift of buckets, and $B = 2^b$ is the number of buckets.  The rest of this section devotes to proving \cref{lem:scalar-main}, which bounds the mean squared quantization error for $\quant$ applied to $X_a$.

\begin{lemma}
\label{lem:scalar-main}
Uniformly over all $d\ge1$ and all $a\in\R^d$ satisfying $\sum_i a_i^2=1$,
\[
    \E_{\eps,U}\bigl(X_a-\quant(X_a)\bigr)^2
    \le \frac{\pi\sqrt3/2+o(1)}{B^2},
    \qquad B\to\infty.
\]
\end{lemma}

As outlined in the proof sketch, the proof proceeds by separating the behavior of the scalar quantizer according to whether
the input lies in a ``central'' interval $[-M_B,M_B]$, where $M_B:=\sqrt{5\log B}$.  We first
show in \cref{lem:rademacher-inputs} that every normalized Rademacher sum has
Gaussian-type tail and moment bounds.  Inside the central interval, \cref{lem:central-linearization}
shows that the random shift $U$ makes the displacement in the $F$-coordinate uniform at scale $1/B$, and that $F^{-1}$ can be replaced by its first-order approximation on this scale.
Outside the central interval, \cref{lem:global-envelope} gives a rough bound on the squared
quantization error; the tail bound from \cref{lem:rademacher-inputs} then makes this
rough estimate negligible.  Combining these two regimes gives \cref{lem:scalar-main}.

The following lemma is standard (see e.g.,~\cite{ailon2009fast}, proof of Lemma 1), but we reprove it here for completeness.

\begin{lemma}
\label{lem:rademacher-inputs}
For every unit vector $a\in\R^d$, the random variable $X_a$ satisfies
\[
    \E e^{\lambda X_a}\le e^{\lambda^2/2}\qquad\text{for all }\lambda\in\R.
\]
Consequently,
\[
    \Prb(|X_a|>s)\le 2e^{-s^2/2}\qquad\text{for all } s>0.
\]
Moreover, if $G\sim \N(0,1)$, then
\[
    \E e^{X_a^2/3}\le \E e^{G^2/3}=\sqrt3.
\]
\end{lemma}

\begin{proof}
By the definition of $X_a$ and the independence between $\eps_i$'s, 
\[
\E e^{\lambda X_a}= \E e^{\lambda \sum_{i=1}^da_i \eps_i} = \prod_{i=1}^d \E e^{\lambda a_i \eps_i}.
\]
For each Rademacher random variable $\eps_i$,
\[\E e^{\lambda a_i \eps_i} = \frac{1}{2}e^{\lambda a_i} + \frac{1}{2}e^{-\lambda a_i} = \cosh{(\lambda a_i)} \leq e^{(\lambda a_i)^2/2}.\]
Thus
\begin{equation}\label{eq:MGF}
    \E e^{\lambda X_a} \leq \prod_{i=1}^d  e^{(\lambda a_i)^2/2} = e^{\sum_{i=1}^d a_i^2 \lambda^2/2} = e^{\lambda^2/2}.
\end{equation}
 The tail bound follows from Markov's inequality applied to both $e^{\lambda X_a}$ and $e^{-\lambda X_a}$ and optimize over $\lambda > 0$.

For the last argument, since $\E_{G\sim \N(0, 1)} e^{\tau G} = e^{\tau^2/2}$  for any $\tau \in \R$, conditioning on $X_a$ and taking $\tau = \sqrt{2/3} X_a$ gives 
\[
    e^{X_a^2/3}=\E_G e^{\sqrt{\frac23}G X_a}.
\]
Let $G$ be independent of $\eps_1, \ldots \eps_d$, 
we obtain
\[
    \E e^{X_a^2/3}
    =\E_G\E_\eps e^{\sqrt{\frac23}G X_a}
    \le \E_G e^{G^2/3}
    =\sqrt3.
\]
where the inequality follows from the earlier MGF bound \cref{eq:MGF}.
\end{proof}

We next state our main lemma for the central event. It controls the quantization error for all real $t$ (in particular, our random variable $X_a$) that lies in the central interval $[-M_B, M_B]$.

\begin{lemma}
\label{lem:central-linearization}
Let $M_B:=\sqrt{5\log B}$.  For all $|t|\le M_B$,
\[
    \E_U\bigl(t-\quant(t)\bigr)^2 \le (1+o_B(1)) \frac{A^2}{12B^2}(2\pi)^{1/3}e^{t^2/3}.
\]
where the $o_B(1)$ term is uniform for  $|t|\le M_B$.
\end{lemma}

Before proving \cref{lem:central-linearization}, we state the following fact about $F$, which will be helpful for our proof.  Let $\Phi$  denote the CDF of the standard normal distribution.

\begin{claim}\label{clm:F-distribution}
    For any $t \in \R$, $F(t) = \Phi(t/\sqrt{3})$.
\end{claim}
\begin{proof}
    Recall that 
\[
    F(t):=\frac1A\int_{-\infty}^t \varphi(s)^{1/3}\dd s.
\]
where $A=\int_\R \varphi(t)^{1/3}\dd t$ is the normalization factor. Since 
\[
\varphi(s)^{1/3} = (2\pi)^{-1/6}e^{-s^2/6}
\]
is proportional to the density of $\N(0, 3)$, $F(t)$ is exactly the CDF of $\N(0,3)$. Thus for $G \sim \N(0, 1)$, 
\[
F(t) = \Prb[\sqrt{3}G \leq t] =  \Prb[G \leq t/\sqrt{3}] = \Phi(t/\sqrt{3}).
\]
\end{proof}

\begin{corollary}\label{cor:F-bound}
    Let $M_B= \sqrt{5 \log B}$. We have $F(-M_B)$, $1-F(M_B)$, $F(M_B+1)-F(M_B)$, and $F(-M_B)-F(-M_B-1) \gg B^{-1}$.
\end{corollary}
\begin{proof}
By \cref{clm:F-distribution} and the symmetry of the Gaussian distribution,
\[
F(-M_B)=1-F(M_B)=\Phi(-M_B/\sqrt 3).
\]
Since $M_B^2=5\log B$, the Gaussian tail estimate gives
\[
\Phi(-M_B/\sqrt 3)=B^{-5/6+o(1)}.
\]
which is in particular $\omega(B^{-1})$. For the two increments,
\begin{align*}
F(M_B+1)-F(M_B)
    &= \int_{M_B}^{M_B+1} F'(t)\dd t \\
    &= \int_{M_B}^{M_B+1}  \frac{(2\pi)^{-1/6}}{A}\exp(-t^2/6) \dd t\\
    &\geq \frac{(2\pi)^{-1/6}}{A}\exp(-(M_B+1)^2/6)
\end{align*}
where the second equality follows from the definition of $F$.
    Since $M_B^2=5\log B$ and $A$ is a constant, this is $B^{-5/6-o(1)} = \omega(B^{-1})$ as claimed. By the symmetry, the same  bound holds for $F(-M_B)-F(-M_B-1)$.
\end{proof}

\begin{proof}[Proof of \cref{lem:central-linearization}]
 Recall that when defining the grid points in \cref{alg:hsgb}, we set $h_1 = (1+U)/B$ and $h_{B-1} = (B-1+U)/B$. As $U$ varies in $[0, 1)$, we always have $h_1 \leq 2/B$ and $h_{B-1} \geq 1- (1/B)$.  On the other hand, since $F$ is increasing, for $|t| \leq M_B$, we have $F(-M_B) \leq F(t) \leq F(M_B)$. Applying \cref{cor:F-bound}, we get
 \[
\omega(B^{-1}) \leq F(t) \leq 1- \omega(B^{-1}).
 \]
 So for all sufficiently large $B$, $h_1 < F(t) < h_{B-1}$, i.e. $F(t)$ never falls in one of the two endpoint buckets. 

Therefore the effect of the random offset $U$ is to uniformly choose the relative position of $F(t)$ inside
a bucket. Equivalently, the midpoint of the bucket containing $F(t)$ can be written as a random variable
\[
F(t)+\frac{T}{B},
\]
where $T$ is uniformly distributed on $[-1/2,1/2]$. And its preimage is
\[
\quant(t)=F^{-1}\left(F(t)+\frac{T}{B}\right).
\]
We next estimate $\quant(t)-t$ using the first-order Taylor expansion of $F^{-1}$ at the point $F(t)$. Formally, we claim that, uniformly over $|t|\le M_B$ and $|T/B|\le 1/(2B)$,
\[
\quant(t) - t= F^{-1}\left(F(t)+\frac{T}{B}\right)-t
=
(1+o_B(1))\frac{T/B}{F'(t)}.
\]
To show this, we observe that for $s \in [0, T/B]$, the point
$F^{-1}(F(t)+s)$ lies in the interval $[-M_B-1,M_B+1]$, for all sufficiently large $B$.
Indeed,
$F(M_B+1)-F(M_B)$ and $
F(-M_B)-F(-M_B-1)\gg B^{-1}$ by \cref{cor:F-bound}.
So for $|s|\le |T/B| \le 1/(2B)$, moving by $s$ in the $F$-coordinate cannot move the preimage
past $M_B+1$ or $-M_B-1$.

Now let $y=F^{-1}(F(t)+s)$. The mean value theorem gives
\[
s=F'(\xi)(y-t)
\]
for some $\xi$ between $y$ and $t$. By the previous argument, both endpoints lie in $[-M_B-1,M_B+1]$, and hence so does $\xi$. Therefore,
\[
|y-t|
\le
|s|\sup_{|\xi|\le M_B+1}\frac{1}{F'(\xi)}
\le
\Theta\left(B^{-1}\exp\left(\frac{(M_B+1)^2}{6}\right)\right)
=
B^{-1/6+o(1)}
\]
where the second step follows from $F'(\xi) = \frac{(2\pi)^{-1/6}}{A}e^{-\xi^2/6}$. This implies
\[
|y^2-t^2|
=
|y-t|\,|y+t|
\le
B^{-1/6+o(1)}\cdot O(M_B)
= B^{-1/6+o(1)}\cdot O(\sqrt{\log B}) = 
o_B(1)
\]
thus
\[
\frac{F'(t)}{F'(y)}
= \frac{e^{-t^2/6}}{e^{-y^2/6}} = 
\exp\left(\frac{y^2-t^2}{6}\right)
=
1+o_B(1)
\]
uniformly over $|t|\le M_B$ and $|s|\le 1/(2B)$.
Equivalently,
\[
\frac{1}{F'(F^{-1}(F(t)+s))}
=
(1+o_B(1))\frac{1}{F'(t)}.
\]
Finally, integrating from $0$ to $T/B$ gives
\begin{equation}\label{eq:linear-approx}
    F^{-1}\left(F(t)+\frac{T}{B}\right)-t
=
\int_0^{T/B}
\frac{1}{F'(F^{-1}(F(t)+s))}\dd s
=
(1+o_B(1))\frac{T/B}{F'(t)}.
\end{equation}
Using $\mathbb E T^2=1/12$ and $F'(t) = \frac{(2\pi)^{-1/6}}{A}e^{-t^2/6}$, we obtain
\[
\mathbb E_U (t-\quant(t))^2
\le
(1+o_B(1))\frac{1}{12B^2}\frac{1}{F'(t)^2} = (1+o_B(1)) \frac{A^2}{12B^2}(2\pi)^{1/3}e^{t^2/3}.
\]
\end{proof}

Next we consider the tail event when $|t|=|X_a| > M_B$. In this event, we use the following crude error bound, which suffices for the entire proof due to the small probability of the tail event.

\begin{lemma}
\label{lem:global-envelope}
There is an absolute constant $C$ such that, for every $B\ge 2$ and every $t\in \mathbb R$,
\[
\mathbb E_U (t-\quant(t))^2 \le C(t^2+\log B).
\]
\end{lemma}

\begin{proof}
Fix the value of $U$. Since $F^{-1}$ is increasing, the reconstruction values
$q_0,\ldots,q_{B-1}$ are increasing. Hence every reconstruction value lies between
the two endpoint reconstruction values, and therefore
\[
|\quant(t)|\le R_B(U),
\]
where
\[
R_B(U):=
\max\left\{
\left|F^{-1}\left(\frac{1+U}{2B}\right)\right|,
\left|F^{-1}\left(1-\frac{1-U}{2B}\right)\right|
\right\}.
\]
By \cref{clm:F-distribution}, we have $F^{-1}(p)=\sqrt 3\,\Phi^{-1}(p)$. Using the
standard Gaussian quantile bound
\[
|\Phi^{-1}(r)|^2+|\Phi^{-1}(1-r)|^2\le C\log(1/r),
\qquad 0<r\le 1/2,
\]
we obtain
\[
R_B(U)^2
\le
C\log\left(\frac{2B}{1+U}\right)
+
C\log\left(\frac{2B}{1-U}\right).
\]
Taking expectation over $U$,
\[
\mathbb E_U R_B(U)^2
\le
C\int_0^1 \log(2B)\,\dd U
+
C\int_0^1 \log\left(\frac{2B}{1-U}\right)\dd U
\le C\log B.
\]
Thus
\[
\mathbb E_U (t-\quant(t))^2
\le
2t^2+2\mathbb E_U q(t)^2
\le
2t^2+2\mathbb E_U R_B(U)^2
\le
C(t^2+\log B).
\]
\end{proof}

With \cref{lem:rademacher-inputs}, \cref{lem:central-linearization}, and \cref{lem:global-envelope}, we can now prove the scalar guarantee.

\begin{proof}[Proof of \cref{lem:scalar-main}]
Fix $a\in \mathbb R^d$ with $\sum_i a_i^2=1$, and write $X=X_a$. We split the
expectation according to the central event $|X|\le M_B$.

We first consider the central contribution, i.e. when $|X| \le M_B$. By \cref{lem:central-linearization}, and using the
uniformity of the $o_B(1)$ term, there exists a deterministic sequence $\alpha_B\to 0$ such that
\begin{align*}
    B^2\mathbb E_{\varepsilon,U}\left[(X-\quant(X))^2\mathbf 1_{\{|X|\le M_B\}}\right]
&\le
(1+\alpha_B)\frac{A^2}{12}(2\pi)^{1/3}
\mathbb E_{\varepsilon} e^{X^2/3} \\
&\le (1+\alpha_B)\frac{A^2}{12}(2\pi)^{1/3}\sqrt 3 \\
&= (1+\alpha_B)\frac{A^3}{12}
\end{align*}
where the second inequality follows from
$\mathbb E_{\varepsilon} e^{X^2/3}\le \sqrt 3$ as shown in \cref{lem:rademacher-inputs}. Since $A=(2\pi)^{1/3}\sqrt 3$, this is $(1+\alpha_B)\frac{\pi\sqrt 3}{2}$. 
It remains to show that the complement of the central event contributes $o(B^{-2})$. We use \cref{lem:global-envelope}:
\[
B^2\mathbb E_{\varepsilon,U}\left[(X-\quant(X))^2\mathbf 1_{\{|X|>M_B\}}\right]
\le
CB^2\mathbb E_{\varepsilon}\left[(X^2+\log B)\mathbf 1_{\{|X|>M_B\}}\right].
\]
Then using the tail bound from \cref{lem:rademacher-inputs},
\[
\mathbb E_{\varepsilon}\left[X^2\mathbf 1_{\{|X|>M_B\}}\right]
=
M_B^2\Pr[|X|>M_B]+\int_{M_B^2}^{\infty}\Pr[X^2>v]\,\dd v
\le
C(M_B^2+1)e^{-M_B^2/2}.
\]
Also,
\[
(\log B)\Pr[|X|>M_B]\le C(\log B)e^{-M_B^2/2}.
\]
Since $M_B^2=5\log B$, this gives
\[
B^2\mathbb E_{\varepsilon,U}\left[(X-\quant(X))^2\mathbf 1_{\{|X|>M_B\}}\right]
\le
CB^2(\log B)e^{-5\log B/2}
=
C(\log B)B^{-1/2}
=
o(1).
\]
Combining the central and tail contributions,
\[
B^2\mathbb E_{\varepsilon,U}(X_a-\quant(X_a))^2
\le
\frac{\pi\sqrt 3}{2}+o(1),
\]
where the $o(1)$ term is uniform over $d$ and $a$.
\end{proof}

\subsection{Proof of \cref{thm:vector-main-biased}}

We now return to the vector quantizer in \cref{alg:hsgb} and prove \cref{thm:vector-main-biased}. The following lemma is the only
property of the randomized Hadamard transform that we need: after the normalization, each coordinate is exactly a Rademacher sum with a unit coefficient vector.

\begin{lemma}\label{lem:Hadamard-Rademacher}
Fix $x\in \calS^{d-1}$, and let $y=HDx$ and $z_i =\sqrt d y_i$
for $i \in [d]$ as in \cref{alg:hsgb}. For every $i\in[d]$, there exists a unit vector
$a^{(i)}\in \mathbb R^d$ such that
$z_i =X_{a^{(i)}}$.
\end{lemma}

\begin{proof}
By the definition of $D$, we have $D=\operatorname{diag}(\varepsilon_1,\ldots,\varepsilon_d)$.
Thus
\[
z_i
=
\sqrt d\,(HDx)_i
=
\sum_{j=1}^d \sqrt d\,H_{i, j}x_j\eps_j.
\]
Define
\[
a_j^{(i)}:=\sqrt d\,H_{i, j}x_j.
\]
Since $H$ is a normalized Hadamard matrix, $H_{i, j}^2=1/d$ for every $i,j$. Hence
\[
\sum_{j=1}^d (a_j^{(i)})^2
=
\sum_{j=1}^d dH_{i, j}^2x_j^2
=
\sum_{j=1}^d x_j^2
=
1.
\]
Therefore $z_i=\sum_ja_j^{(i)}\varepsilon_j=X_{a^{(i)}}$, as claimed.
\end{proof}

With this, we can directly apply our scalar estimate for each coordinate and obtain the mean squared error bound for the vector quantizer.

\begin{proof}[Proof of \cref{thm:vector-main}]
Fix $x\in \calS^{d-1}$. Let $y=HDx$, and let $z_i=\sqrt d\,y_i$ be the scalar input
quantized in coordinate $i$. If $\widetilde y$ denotes the vector reconstructed in the
Hadamard domain, then, by the scaling in \cref{alg:hsgb},
\[
\widetilde y_i=\frac{1}{\sqrt d}\quant(z_i).
\]
Since $H$ and $D$ are orthogonal,
\[
\|x-\widetilde x\|_2^2
=
\|HDx-\widetilde y\|_2^2
=
\frac{1}{d}\sum_{i=1}^d (z_i-\quant(z_i))^2.
\]
Taking expectation over $D$ and $U$, and using linearity of expectation,
\[
\mathbb E_{D,U}\|x-\widetilde x\|_2^2
=
\frac{1}{d}\sum_{i=1}^d
\mathbb E_{D,U}(z_i-\quant(z_i))^2.
\]
By the \cref{lem:Hadamard-Rademacher}, each $z_i$ is equal to $X_{a^{(i)}}$ for a unit vector
$a^{(i)}$. Applying \cref{lem:scalar-main} to each coordinate gives
\[
\mathbb E_{D,U}(z_i-\quant(z_i))^2
\le
\frac{\pi\sqrt 3/2+o(1)}{B^2}.
\] 
Then averaging over $i$, we obtain
\[
\mathbb E_{D,U}\|x-\widetilde x\|_2^2
\le
\frac{\pi\sqrt 3/2+o(1)}{B^2}.
\]
\end{proof}

\subsection{An unbiased variant of \cref{alg:hsgb}}\label{subsec:unbiased}

In this section we modify the scalar quantizer, in particular the scalar reconstruction rule, so
that the resulting vector estimator is unbiased, while preserving the same
asymptotic distortion bound. Recall that
\[
F(t)=\frac{1}{A}\int_{-\infty}^t \phi(s)^{1/3}\,ds,
\qquad
A=\int_{\mathbb R}\phi(s)^{1/3}\,ds.
\]
Previously in \cref{alg:hsgb}, we define buckets with $1/B$ spacing in the $F$-coordinates, and reconstruct each bucket at the inverse image ($F^{-1}$) of its midpoint. In this section, we will use buckets with spacing $1/(B-1)$, and we will replace $F^{-1}$ with a   reconstruction function
$G:[-\frac{1}{2(B-1)},1+\frac{1}{2(B-1)}]\to\R$
satisfying
\begin{equation}\label{eq:G}
    (B-1)
\int_{r-\frac{1}{2(B-1)}}^{r+\frac{1}{2(B-1)}} G(s) \dd s
=
F^{-1}(r)
\qquad\text{for every }r\in(0,1).
\end{equation}
Moreover, we will choose $G$ such that in the central region $[-\sqrt{5 \log B}, \sqrt{5 \log B}]$, $G$ is close to $F^{-1}$. Formally:

\begin{lemma}\label{lem:G-requirements}
    There exists an explicitly defined function $G:[-\frac{1}{2(B-1)},1+\frac{1}{2(B-1)}]\to\R$ that satisfies \cref{eq:G}. 
    Moreover, it can be chosen so that for all $|t|\le \sqrt{5\log B}$ and all $|v|\le 1/2$,
\begin{equation}\label{eq:central-close}
    G\left(F(t)+\frac{v}{B-1}\right)=t+\frac{v}{(B-1)F'(t)}+
o(1)\frac{1}{(B-1)F'(t)}.
\end{equation}
\end{lemma}
\begin{proof}
We denote the spacing of the grid buckets $\delta := \frac{1}{B-1}$. 
We start by defining $G$  on the open interval $I := (-\frac{\delta}{2},  1+\frac{\delta}{2})$ as follows:  for 
$u\in(\frac{1-\delta}{2},\frac{1+\delta}{2}]$ and $k \in \Z$, whenever
$u+k\delta \in I$, we define $G$ as

\[
G(u+k\cdot \delta)
:=F^{-1}(u)+
\begin{cases}
    \delta\sum_{j=0}^{k-1}(F^{-1})'\bigl(u+(j+1/2)\cdot \delta \bigr)  & \text{ if }  k>0 \\
    0 & \text{ if } k = 0 \\
    -\delta\sum_{j=k}^{-1}(F^{-1})'\bigl(u+(j+1/2)\cdot \delta \bigr)  & \text{ if }  k<0 \\
\end{cases} .
\]
The key idea behind this definition is to enforce the recurrence
\[
G(u+(k+1)\delta)-G(u+k\delta)
=
\delta (F^{-1})'\bigl(u+(k+1/2)\delta\bigr).
\]
We first argue that $G$ satisfies \cref{eq:G}. For every $r \in(0, 1)$, there is a unique pair of $u$ and $k$ such that we can write $r-\delta/2 = u+k\cdot \delta$, it follows that
$G(r+\delta/2)-G(r-\delta/2) =\delta (F^{-1})'(r).
$
Therefore,
\[
\frac{\dd }{ \dd r} \left( \frac{1}{\delta}
\int_{r-\delta/2}^{r+\delta/2}G(s) \dd s \right) = \frac{1}{\delta}\left(G(r+\delta/2)-G(r-\delta/2)\right)  = (F^{-1})'(r),
\]
i.e.  
\[C := \frac{1}{\delta}
\int_{r-\delta/2}^{r+\delta/2}G(s) \dd s - F^{-1}(r)\]
is constant in $r$. Evaluating this at $r=1/2$ shows that $C = 0$. Indeed, on $(\frac{1-\delta}{2}, \frac{1+\delta}{2}]$, $G = F^{-1}$, and $F(t) = \Phi(t/\sqrt{3})$ (\cref{clm:F-distribution}) implies that $F^{-1}(1/2) = 0$ and $F^{-1}(1-r)=-F^{-1}(r)$ for every $r$.

This shows that $G$ satisfies \cref{eq:G} on the open interval $I$. We extend it arbitrarily to the endpoints for notational convenience in the codebook. These endpoint values are irrelevant for \cref{eq:G} since they are measure-zero for the integral.

It remains to show \cref{eq:central-close}. We first compare $G$ with $F^{-1}$.
Fix $r\in I$, and write $r=u+k\delta$ again with the unique pair of $u\in(\frac{1-\delta}{2},\frac{1+\delta}{2}]$ and $k \in \Z$. In the following, we focus on  $k 
\geq 0$, with the convention that $\sum^{-1}_{j=0 } (\cdot )= 0$. The case $k <0$ holds symmetrically.  By the definition
of $G$,
\[
G(r)-F^{-1}(r)
=
\sum_{j=0}^{k-1}
\left[
\delta (F^{-1})'\bigl(u+(j+1/2)\delta\bigr)
-
\int_{u+j\delta}^{u+(j+1)\delta} (F^{-1})'(s)\dd s
\right].
\]
Hence, by the midpoint rule error bound,
\[
|G(r)-F^{-1}(r)|
\le
O\left(\delta^3
\sum_{j=0}^{k-1}
\sup_{s\in [u+j\delta,u+(j+1)\delta]}
\left|(F^{-1})'''(s)\right|\right).
\]
Equivalently, for $J$ denoting the interval between $u$ and $r$, 
\begin{equation*}
|G(r)-F^{-1}(r)|
\le
O\left(\delta^2\int_{J}\left|(F^{-1})'''(s)\right|\dd s\right)
+
O\left(\delta^3\sup_{s\in J}\left|(F^{-1})'''(s)\right|\right).    
\end{equation*}

It remains to bound this RHS plus $F^{-1}(r) - t$, using our parameters from \cref{eq:central-close}: 
\[r=F(t)+v\delta, \qquad |t|\le M_B, \qquad  |v|\le 1/2.\]
In \cref{lem:central-linearization} (\cref{eq:linear-approx}),  we have shown  that $F^{-1}(r) - t = (1+ o(1))\frac{v\delta}{F'(t)}$. Also, we have shown that
$F^{-1}(r) \in [-M_B-1,M_B+1]$ for all sufficiently large $B$. Combining this with the fact that $F^{-1}(u)=O(\delta)$ for $u\in(\frac{1-\delta}{2}, \frac{1+\delta}{2}]$, we have that for every $s \in J$, $|F^{-1}(s)|\leq M_B+1$. And by the definition of $F$, one can observe that
\[
(F^{-1})'(s)=\frac1{F'(F^{-1}(s))},\qquad
(F^{-1})''(s)=\frac{F^{-1}(s)}{3F'(F^{-1}(s))^2},\]
and\[
(F^{-1})'''(s)
=
\left(\frac13+\frac{2(F^{-1}(s))^2}{9}\right)\frac1{F'(F^{-1}(s))^3}.
\]
Combining all these, we get
\[
\int_{J}\left|(F^{-1})'''(s)\right|\dd s
\le
O\left(\frac{M_B}{F'(t)^2}\right),
\qquad
\sup_{s\in J}\left|(F^{-1})'''(s)\right|
\le
O\left(\frac{M_B^2}{F'(t)^3}\right).
\]
Returning to $\cref{eq:central-close}$, we get that
\begin{align*}
    G(r) -t &\leq F^{-1}(r) - t + O\left(\delta^2\int_{J}\left|(F^{-1})'''(s)\right|\dd s\right)
+
O\left(\delta^3\sup_{s\in J}\left|(F^{-1})'''(s)\right|\right) \\
&\leq (1+o(1))\frac{v\delta}{F'(t)} + \delta^2O\left(\frac{M_B}{F'(t)^2}\right) + \delta^3 O\left(\frac{M_B^2}{F'(t)^3}\right) \\
&= (1+o(1))\frac{v\delta}{F'(t)} + \poly \log B \cdot ( B^{-1/6 +o(1)} + B^{-1/3 +o(1)} )\cdot \frac{\delta}{F'(t)}\\
&= \frac{v\delta}{F'(t)} + o(1)\frac{\delta}{F'(t)}
\end{align*}
where in the third step we used $M_B = \sqrt{5 \log B}$ and $\frac{\delta}{F'(t)}
\le
\Theta(\delta \exp(M_B^2/6))
=
B^{-1/6+o(1)}$.
\end{proof}

We now describe the unbiased variant of \cref{alg:hsgb} in \cref{alg:hsgb-unbiased}. The only change is the procedure of \textsc{ConstructScalarCodebook}, so we omit repeating the procedures $\textsc{VectorQuant}$ and $\textsc{VectorDequant}$.

\begin{algorithm}
\small
\begin{algorithmic}[1]
\caption{Unbiased dithered scalar Gaussian quantization}
\label{alg:hsgb-unbiased}
\State \textbf{input:} dimension $d$ and bit-width $b$
\State Let $H\in\R^{d\times d}$ be a normalized Hadamard matrix  and let $B\gets 2^b$.
\State Sample independent Rademacher random variables $\eps_1,\ldots,\eps_d$ and set $D\gets\diag(\eps_1,\ldots,\eps_d)$.
\State Sample an offset $U\sim\Unif[0,1)$.

\Statex \hrulefill

\Procedure{ConstructScalarCodebook}{$B, U$}
\State Define buckets $\mathcal B_j\gets\{t\in\R: \lfloor (B-1)F(t) - U \rfloor + 1 = j\}$  for $j=0,\ldots,B-1$.
\State Define reconstruction values $q_j\gets G (\frac{j + U -1/2}{B-1})$ for $j=0,\ldots,B-1$.
\EndProcedure

\end{algorithmic}
\end{algorithm}

In English, we use a shifted grid in the
$F$-coordinate with spacing $1/(B-1)$. For $r\in(0,1)$, define grid points
\[
h(r)
=
\left\lfloor (B-1)r-U\right\rfloor ,
\]
where $U\sim \operatorname{Unif}[0,1)$. We only need to focus on $r \in (0, 1)$ since later we will take $r = F(t)$, which is in $(0, 1)$ for any real $t$. Since $0< r < 1$, the index satisfies
$h(r)\in\{-1,0,1,\ldots,B-2\}$.
Thus $h(r)+1\in\{0,1,\ldots,B-1\}$, so the index is still encodable using
$\log_2 B=b$ bits.
We denote the midpoint of the shifted cell containing $r$ 
\[
m(r)
=
\frac{h(r)+U+1/2}{B-1}.
\]
Then for every fixed $r\in(0,1)$,
\[
m(r)
=
r+\frac{V}{B-1},
\]
where $V$ is a random variable $
V\sim \operatorname{Unif}\left[-\frac12,\frac12\right].$ 
In the following, we again denote the scalar quantizer:
\[\quant(t) := q_j \qquad \text{when $t \in \mathcal{B}_j$}.\]
Then by definition, 
\begin{align}\label{eq:uniform-unbiased}
    \quant(t) &= G\left(\frac{\lfloor (B-1)F(t)-U\rfloor+U+1/2}{B-1}\right) \notag \\
    &= G\left(F(t)+\frac{V}{B-1}\right),  \qquad V\sim \operatorname{Unif}\left[-\frac12,\frac12\right]
\end{align}

 Combining this with the requirements on $G$ (\cref{lem:G-requirements})  gives both the unbiasedness and the scalar estimate bound. 
\begin{lemma}
\label{lem:scalar-unbiased}
For every $t\in\mathbb R$, $\E_U[\quant(t)] = t$.
Moreover, let
$X_a=\sum_{i=1}^d a_i\eps_i$
where $\varepsilon_1,\ldots,\varepsilon_d$ are independent Rademacher random
variables and $a\in\mathbb R^d$  be a unit vector. Then,
uniformly over $d\ge1$ and $a$,
\[
\mathbb E_{\eps,U}
\bigl(X_a-\quant(X_a)\bigr)^2
\le
\left(\frac{\pi\sqrt3}{2}+o(1)\right)B^{-2}.
\]
\end{lemma}
\begin{proof}
We denote the spacing of grid buckets $\delta:=\frac{1}{B-1}$. Fix $t\in\R$. By \cref{eq:G} and \cref{eq:uniform-unbiased},
\[
\E_U[\quant(t)]
=\frac{1}{\delta}\int_{F(t)-\delta/2}^{F(t)+\delta/2}G(s)\dd s
=F^{-1}(F(t))=t .
\]
This proves the unbiasedness.

It remains to prove the mean squared error bound, which parallels the proof to \cref{lem:scalar-main}. Let $M_B:=\sqrt{5\log B}$. We first consider the contribution from $|X|\le M_B$. By \cref{lem:G-requirements},(\cref{eq:central-close}), uniformly over $|t|\le M_B$ and $|v|\le 1/2$,
\[
G(F(t)+v\delta)
=
t+\frac{v\delta}{F'(t)}
+
o(1)\frac{\delta}{F'(t)} .
\]
In our case, $v=V$ is uniform on $[-\frac{1}{2}, \frac{1}{2}]$, therefore, $\E V^2 = \frac{1}{12}$ and thus
\begin{align*}
\E_U(t-\quant(t))^2 &= \E_U(t-G(F(t)+V\delta))^2 \\
&= \frac{\delta^2}{F'(t)^2} \E(V+o(1))^2\\
&\le
(1+o(1))\frac{\delta^2}{12F'(t)^2},    
\end{align*}

uniformly for $|t|\le M_B$. Now we take the expectation over $X_a$. Using $F'(t)=A^{-1}(2\pi)^{-1/6}e^{-t^2/6}$, $\delta = \frac{1}{B-1}$, and $\E_\eps e^{X_a^2/3} \leq \sqrt{3}$, this becomes
\begin{align*}
    \E_{\eps,U}\left[(X_a-\quant(X_a))^2\mathbf 1_{\{|X_a|\le M_B\}}\right]
&\le
(1+o(1))\frac{\delta^2A^2}{12}(2\pi)^{1/3}\E_\eps e^{X_a^2/3} \\
&= \left(\frac{\pi\sqrt3}{2}+o(1)\right)B^{-2}
\end{align*}
which concludes the  central contribution. We now bound the tail contribution. Similar to \cref{lem:global-envelope}, we use a crude bound following from the construction of $G$. For every $t\in\mathbb R$ and every $|v|\le 1/2$,
\[
|G(F(t)+v\delta)|
\le
C\left(1+|t|+\frac{\delta}{F'(t)}\right).
\]
where $C$ is a constant. Indeed, suppose first that $t\ge 0$. In the definition of $G$ in \cref{lem:G-requirements}, the values are obtained from the interval around $1/2$ by summing midpoint values of $(F^{-1})'$. Since $(F^{-1})'$ is increasing on $[1/2,1)$, the sum up to $F(t)+v\delta$ is bounded by the corresponding integral up to $F(t)$, plus one endpoint term. This gives $C(1+t+\delta/F'(t))$. The case $t<0$ is symmetric, using that $F^{-1}(1-s)=-F^{-1}(s)$ and that $(F^{-1})'$ is symmetric around $1/2$.

Consequently,
\[
\E_U(t-\quant(t))^2
\le
C\left(1+t^2+\frac{\delta^2}{F'(t)^2}\right).
\]
Applying this with $t=X_a$ gives
\begin{align*} 
&\E_{\eps,U}\left[(X_a-\quant(X_a))^2\mathbf 1_{\{|X_a|>M_B\}}\right] \\
&\le
O\left(\E_\eps\left[(1+X_a^2)\mathbf 1_{\{|X_a|>M_B\}}\right]
+
\E_\eps\left[e^{X_a^2/3}\mathbf 1_{\{|X_a|>M_B\}}\right]B^{-2}\right).
\end{align*}
Using \cref{lem:rademacher-inputs} and the same argument in \cref{lem:scalar-main}, this tail contribution is at most
\[
O( \poly \log B \cdot e^{-5\log B/2}) = o(1)B^{-2}.
\]
Combining the central and tail estimates, we obtain
\[
\E_{\eps,U}(X_a-\quant(X_a))^2
\le
\left(\frac{\pi\sqrt 3}{2}+o(1)\right)B^{-2}.
\]
uniformly over all $d\ge 1$ and all unit vectors $a\in\mathbb R^d$. 
\end{proof}

Using \cref{lem:scalar-unbiased}, we immediately obtain our main theorem (\cref{thm:vector-main}), that $\widetilde x$ is unbiased with the claimed mean squared error bounded.

\begin{proof}[Proof of \cref{thm:vector-main}]
    Regarding the mean squared error bound, as in the proof of \ref{thm:vector-main-biased}, fix $x\in \calS^{d-1}$, we have that
\[
\mathbb E_{D,U}\|x-\widetilde x\|_2^2
=
\frac{1}{d}\sum_{i=1}^d
\mathbb E_{D,U}(z_i-\quant(z_i))^2.
\]
where  $z_i=\sqrt d y_i = \sqrt d (HDx)_i$ follows the distribution of $X_a$ for some unit vector $a$.
 Applying  \cref{lem:scalar-unbiased} to each coordinate gives
\[
\mathbb E_{D,U}(z_i-\quant(z_i))^2
\le
\frac{\pi\sqrt 3/2+o(1)}{B^2}.
\] 
Then averaging over $i$, we obtain
\[
\mathbb E_{D,U}\|x-\widetilde x\|_2^2
\le
\frac{\pi\sqrt 3/2+o(1)}{B^2}.
\]

For the unbiasedness, we condition on $D$. Once $D$ is fixed, the vector $y=HDx$
and hence all $z_i=\sqrt d\,y_i$ are deterministic. By the unbiasedness part of \cref{lem:scalar-unbiased},
\[
\E_U[\quant(z_i)\mid D]=z_i .
\]
Thus
\[
\E_U[\widetilde y_i\mid D]
=
\frac1{\sqrt d}\mathbb E_U[\quant(z_i)\mid D]
=
\frac{z_i}{\sqrt d}
=
y_i .
\]
Since this holds for every coordinate,
\[
\mathbb E_U[\widetilde y\mid D]=y=HDx .
\]
Therefore
\[
\E_U[\widetilde x\mid D]
=
DH^\top \mathbb E_U[\widetilde y\mid D]
=
DH^\top HDx =x.
\]
Finally, taking expectation over $D$ gives
$\E_{D,U}[\widetilde x]
=
x$.
\end{proof}

\section{Analysis of the inner product quantizer} \label{app:inner-product}

We present the analysis of the inner product quantizer in this section, proving Theorem~\ref{thm:inner-product-error}. Similarly to how the analysis of our vector quantizer relied on a moment generating function bound for Rademacher random variables, the analysis of the inner product quantizer hinges on the following fourth moment bound satisfied by Rademacher random variables:

\begin{lemma}[Mixed fourth moment] \label{lm:fourth-moment}
Let $\eta_1,\ldots,\eta_d$ be independent Rademacher signs.  Let
\[
  X=\sum_{j=1}^d a_j\eta_j,
  \qquad
  Y=\sum_{j=1}^d b_j\eta_j,
\]
where $\norm{a}=\norm{b}=1$.  Then
\[
  \E[X^2Y^2]
  =
  1+2\ip{a}{b}^2-2\sum_{j=1}^d a_j^2b_j^2
  \le 3.
\]
Consequently,
\[
  \E\bigl[X^2(A^2+3Y^2)\bigr]
  \le A^2+9.
\]
\end{lemma}
\begin{proof}
Write
\[
  X^2
  =
  \sum_j a_j^2
  +
  2\sum_{j<k} a_j a_k \eta_j\eta_k
  =
  1+
  2\sum_{j<k} a_j a_k \eta_j\eta_k,
\]
and similarly
\[
  Y^2
  =
  1+
  2\sum_{j<k} b_j b_k \eta_j\eta_k.
\]
Multiplying and taking expectations, all terms vanish except those with the
same pair $\{j,k\}$.  Therefore
\[
  \E[X^2Y^2]
  =
  1+4\sum_{j<k} a_j a_k b_j b_k.
\]
Since
\[
  \ip{a}{b}^2
  =
  \sum_j a_j^2b_j^2
  +
  2\sum_{j<k} a_j a_k b_j b_k,
\]
we obtain
\[
  4\sum_{j<k} a_j a_k b_j b_k
  =
  2\ip{a}{b}^2
  -
  2\sum_j a_j^2b_j^2.
\]
This proves the identity.  The upper bound follows because
$\ip{a}{b}^2\le 1$ and $\sum_j a_j^2b_j^2\ge0$.

Finally, since $\E X^2=1$,  for any $\kappa > 1$, we have
\[
  \E\bigl[X^2(A^2+(\kappa^2-1)Y^2)\bigr]
  \le
  A^2+3(\kappa^2-1)
\]
Setting $\kappa=2$ gives the conclusion.
\end{proof}

\paragraph{Roadmap.} In what follows we prove an upper bound on the rate of the quantizer, i.e. show that it uses  constant number of bits per coordinates, in Section~\ref{sec:rate-bound}.  We then analyze the error of the quantizer and  prove Theorem~\ref{thm:inner-product-error}  in Section~\ref{sec:res-error-bound}.
\subsection{Rate Bound}\label{sec:rate-bound}

In this section, we show that the output of the residual quantizer from Algorithm~\ref{alg:residual_quantizer} can be stored in $O(d)$ bits. 

\begin{lemma}\label{lem:bits-usage}
    Let $B\ge2$ and let $r\in\R^d$ such that $\|r\|_2 \leq 2$.
Run the \textsc{ResidualQuant} procedure of \Cref{alg:residual_quantizer}, using \cref{alg:scalar_quantization}
for the residual scale. Then the output $\idx_\sigma, \ell, \lambda$ can be stored using
$\left(3+\frac{1}{2\ln 2}\right)d + O(\log\log(dB))$ total bits.
\end{lemma}
\begin{proof}
    

We consider the non-trivial case $\idx_\sigma > 0$. Recall that $v = HD r$. Also recall that \textsc{ScalarDequant} outputs $\sigma = \tau_B 2^{\idx_{\sigma}-1} = \tau_B 2^{\lceil \log_2 (s/\tau_B)\rceil }$. By this definition, $\sigma \geq s := \frac{\|r\|_2}{\sqrt{d}}$. 
Therefore, if we write $ Z_i:=\frac{|v_i|}{\sigma}$, then
\[
  \sum_{i=1}^d Z_i^2
  =
  \frac{\norm{HDr}_2^2}{\sigma^2}
  =
  \frac{\norm r_2^2}{\sigma^2}
  \le d
\]
where the second step follows from the fact that $HD$ is orthogonal. 
Then by the definition of the level,
\[
  \ell_i\le 1+\log_2^+(Z_i),
\]
where $\log_2^+(s)=\max\{0,\log_2 s\}$. Since
$\log s\le s^2/2$ for $s\ge1$,
\[
  \log_2^+(Z_i)\le \frac{Z_i^2}{2\ln 2}.
\]
Thus
\[
  \sum_{i=1}^d(\ell_i+1)
  \le
  \left(2+\frac{1}{2\ln 2}\right)d.
\]

This bounds the bit usage of recording the levels $\ell_1, \ldots, \ell_d$. Finally, when $\|r\|_2 \leq 2$, $s = \frac{\|r\|_2}{\sqrt{d}} \in [0, \frac{2}{\sqrt{d}}]$, thus as mentioned in \cref{sec:inner-product}, $\idx_\sigma$ can be encoded using 
\[\lceil\log_2(\lceil \log_2 (\frac{2}{\sqrt{d}}/\tau_B) \rceil +2)\rceil = O(\log\log(dB))\] 
bits.
$\lambda_1, \ldots, \lambda_d$ can be stored trivially in $d$ total bits. Combining these give the claimed bound.

\end{proof}

\begin{remark}[Expected bits per level]
    \cref{lem:bits-usage} gives a deterministic $O(d)$ bound on the total number of bits. Alternatively, we can also show that each level $\ell_i$ can be stored in $O(1)$ bits {\it in expectation}.
    To that effect,  define
\[
  Y_i=\frac{v_i}{s}
  =
  \frac{(HDr)_i}{s}.
\]
The coefficient vector of $Y_i$ has norm one, since
$|H_{ij}|=d^{-1/2}$ and $s=\norm{r}/\sqrt d$. Hence
\[
  \Prb(|Y_i|>t)\le 2e^{-t^2/2}.
\]
For $k\ge1$, the event $\{\ell_i\ge k\}$ implies
\[
  |v_i|>\sigma 2^{k-1}\ge s2^{k-1}.
\]
Therefore
\[
  \Prb(\ell_i\ge k)
  \le
  2\exp\left(-\frac{2^{2(k-1)}}{2}\right),
\]
so the expected unary length of each level is bounded by an absolute constant.
\end{remark}

\begin{remark}[Entropy coding the levels]
One could replace the unary encoding of the levels $\ell_i$ by a Huffman (or other
entropy) code.  This only gives a small improvement: since
\[
  \Pr(\ell_i\ge k)\le 2\exp\left(-2^{2(k-1)}/2\right)\qquad (k\ge2),
\]
the level distribution has doubly-exponential tails, and its entropy is bounded
by a small absolute constant.  In particular,
\[
  H(\ell_i)\le \frac{2}{e\ln2}
  +\sum_{k\ge2} a_k\log_2(1/a_k)<1.579,
  \qquad
  a_k=2\exp\left(-2^{2(k-1)}/2\right).
\]
Thus a prefix code would use, on average, fewer than $(H(\ell_i)+1)+1<3.579$ bits per
coordinate including the sign bit, compared with the unary deterministic bound
$3+1/(2\log 2)\approx3.72$.
\end{remark}

\subsection{Error Bounds for the Residual Quantizer} \label{sec:res-error-bound}

In the next lemma, we bound the error of the residual quantizer (Algorithm~\ref{alg:residual_quantizer}). This guarantee will  later be used in the analysis of our inner product quantizer (Algorithm~\ref{alg:inner_product_quantizer}).

\begin{lemma}
\label{thm:residual-compressor} Let $B\ge2$ and let $r,y\in\R^d$.
Run the quantization procedure of \Cref{alg:residual_quantizer}, using \cref{alg:scalar_quantization}
for the residual scale. Then
\[
  \E |\ip{y}{\widehat r-r}|^2
  \le
  13\frac{\norm{y}^2(\norm{r}^2+B^{-2})}{d}.
\]
\end{lemma}

\begin{proof}
If $y=0$, the claim is immediate. So we consider the non-trivial case where $y \neq 0$. Let $s = \|r\|/\sqrt{d}$. Also, let $\idx_\sigma$ be the output of
$\textsc{ScalarQuant}(s)$ and let $\sigma$ be the output of $\textsc{ScalarDequant}(\idx_\sigma)$. If $\sigma=0$, then $\widehat r=0$ and
$s <1/(dB)$, so $\norm{r}^2<1/(dB^2)$. Hence
\[
  |\ip{y}{\widehat r-r}|^2
  \le \norm{y}^2\norm{r}^2
  \le \frac{\norm{y}^2}{dB^2}
  \le
  13\frac{\norm{y}^2(\norm{r}^2+B^{-2})}{d}.
\]

Now assume $\sigma>0$. Set $ A=\frac{\sigma}{s}$.
As mentioned in \cref{sec:inner-product}, by the definition of $\sigma$ we have $1\le A<2$. Let $u=HDy$ and $v=HDr$. Conditional on $D,r$, the
variables $q_i-v_i$ are independent and centered, and
\[
  \E\left[|\ip{y}{\widehat r-r}|^2\mid D,r\right]
  =
  \sum_{i=1}^d u_i^2(R_i^2-v_i^2).
\]
Define
\[
  X_i=\frac{\sqrt d\,u_i}{\norm{y}},
  \qquad
  Y_i=\frac{v_i}{s}.
\]
Since $R_i=\sigma2^{\ell_i}=A s2^{\ell_i}$ and $\ell_i$ is chosen to be the smallest level such that $|v_i| \leq R$, either $\ell_i = 0$ in which case $R_i = As$, or $\ell_i \geq 1$, in which case $|v_i| > As 2^{\ell_i-1}$ and multiplying by $2$ gives $R_i < 2 |v_i|$. In both cases, we have
\[
  R_i^2-v_i^2
  \le
  s^2\bigl(A^2+(2^2-1)Y_i^2\bigr).
\]
Therefore
\[
  \E |\ip{y}{\widehat r-r}|^2
  \le
  \frac{\norm{y}^2\norm{r}^2}{d^2}
  \sum_{i=1}^d
  \E\bigl[X_i^2\bigl(A^2+3Y_i^2\bigr)\bigr].
\]
For each $i$, the coefficient vectors defining $X_i$ and $Y_i$ have norm one,
so the mixed fourth-moment lemma gives
\[
  \E\bigl[X_i^2\bigl(A^2+3Y_i^2\bigr)\bigr]
  \le A^2+9.
\]
Since $A<2$, this yields
\[
  \E |\ip{y}{\widehat r-r}|^2
  \le
  13\frac{\norm{y}^2\norm{r}^2}{d}
  \le
  13\frac{\norm{y}^2(\norm{r}^2+B^{-2})}{d}.
\]
\end{proof}

\subsection{Proof of \cref{thm:inner-product-error}}

We now prove Theorem~\ref{thm:inner-product-error}, restated here for convenience of the reader:

\begin{theorem}[Repeating \cref{thm:inner-product-error}] Let $x \in \calS^{d-1}$ be a unit vector and let
$y \in \R^d$. Let $\widehat{x}$ be the reconstructed vector produced by \cref{alg:inner_product_quantizer} with a base bit-width $b$.
Then
\[
    \E |\ip{y}{\widehat{x} - x}|^2
    \le
    13\left(\frac{\pi\sqrt{3}}{2}+1+o(1)\right)
    \frac{\norm{y}^2}{d \cdot 4^b},
\]
where the $o(1)$ term vanishes as $b \to \infty$, uniformly over $d$ and $x$.

Moreover, the output of \textsc{Quantize}$(x)$ from \cref{alg:inner_product_quantizer} can be stored using 
\[d b + (3 + \frac{1}{2\ln 2})d + O(\log(b + \log d))\]
total bits. 
\end{theorem}

\begin{proof}
We write $B=2^b$. Recall that the two-stage quantizer outputs
$\widehat{x}=\widetilde{x}+\widehat r$, where
$r=x-\widetilde{x}$. Condition on the randomness of the first stage, so
that $r$ is fixed. By \cref{thm:residual-compressor},
\[
    \E_{D_{\mathrm{res}}}
    \left[|\ip{y}{\widehat r-r}|^2\mid D_{\mathrm{base}},U\right]
    \le
    13\frac{\norm{y}^2(\norm{r}^2+B^{-2})}{d}.
\]
Taking expectation over the first stage gives
\[
    \E |\ip{y}{\widehat{x}-x}|^2
    \le
    13\frac{\norm{y}^2}{d}
    \left(\E_{D_{\mathrm{base}},U}\norm{x-\widetilde x}_2^2+B^{-2}\right).
\]

If we use $\overline x$ to denote $\widetilde x$ before projecting to the unit ball via $\Pi$, then by \cref{thm:vector-main-biased}, we have
\[
    \E_{D_{\mathrm{base}},U}\norm{x-\overline x}_2^2
    \le
    \frac{\pi\sqrt3/2+o_B(1)}{B^2}.
\]
Since the unit ball is a closed convex set and $x \in \mathcal{S}^{d-1}$, point-wise it always hold that
\[\|x - \widetilde x \|^2_2 = \|x - \Pi \overline x \|^2_2 \leq \|x - \overline x\|^2_2\]
thus 
\[
\E_{D_{\mathrm{base}},U}\norm{x-\widetilde x}_2^2
    \le
    \frac{\pi\sqrt3/2+o_B(1)}{B^2}.
\]
Substituting this bound proves the theorem. The total bit usage is $db$ bits for $\idx$ plus the bit usage from \cref{lem:bits-usage}. Note that we have $\|r\| \leq 2$ because after projecting onto the unit ball, both $\|x\|, \|\widetilde x\|\leq 1$.

\end{proof}

\section{Additional experiment results}\label{app:experiment}

\begin{table}[t]
  \centering
  \caption{Detailed NMSE and encode-decode time.  Entries of the form
  $x/y$ give TurboQuant/DASHQuant.  NMSE ratios below $1$ and timing speedups
  above $1$ favor DASHQuant.  Times are in milliseconds for all 100,000
  database vectors; all entries are means over 50 quantizer seeds.}
  \label{tab:nmse-timing}
\scriptsize
\setlength{\tabcolsep}{10pt}
\renewcommand{\arraystretch}{1.1}
\resizebox{0.9\linewidth}{!}{%
  \begin{tabular}{lrrrrr}
    \hline
    Dataset & $b$ & NMSE (TQ/DQ) & DQ/TQ & Time (TQ/DQ) & Speedup \\
    \hline
    GloVe-200 & 2 & $9.424\mathrm{e}{-2}/1.676\mathrm{e}{-1}$ & 1.778 & $10.81/13.88$ & $0.78\times$ \\
              & 3 & $2.706\mathrm{e}{-2}/3.867\mathrm{e}{-2}$ & 1.429 & $10.76/13.77$ & $0.78\times$ \\
              & 4 & $7.385\mathrm{e}{-3}/9.016\mathrm{e}{-3}$ & 1.221 & $10.75/13.80$ & $0.78\times$ \\
              & 5 & $1.942\mathrm{e}{-3}/2.155\mathrm{e}{-3}$ & 1.110 & $11.00/13.86$ & $0.79\times$ \\
              & 6 & $5.350\mathrm{e}{-4}/5.253\mathrm{e}{-4}$ & 0.982 & $10.93/13.79$ & $0.79\times$ \\
              & 7 & $1.906\mathrm{e}{-4}/1.296\mathrm{e}{-4}$ & 0.680 & $10.78/13.89$ & $0.78\times$ \\
              & 8 & $7.458\mathrm{e}{-5}/3.217\mathrm{e}{-5}$ & 0.431 & $5.24/7.99$ & $0.66\times$ \\
    \hline
    DBpedia-1536 & 2 & $9.150\mathrm{e}{-2}/1.641\mathrm{e}{-1}$ & 1.794 & $45.63/17.94$ & $2.54\times$ \\
                 & 3 & $2.619\mathrm{e}{-2}/3.733\mathrm{e}{-2}$ & 1.425 & $46.14/18.12$ & $2.55\times$ \\
                 & 4 & $7.143\mathrm{e}{-3}/8.701\mathrm{e}{-3}$ & 1.218 & $46.76/19.01$ & $2.46\times$ \\
                 & 5 & $1.879\mathrm{e}{-3}/2.086\mathrm{e}{-3}$ & 1.110 & $47.41/19.74$ & $2.40\times$ \\
                 & 6 & $5.176\mathrm{e}{-4}/5.096\mathrm{e}{-4}$ & 0.985 & $47.94/20.00$ & $2.40\times$ \\
                 & 7 & $1.848\mathrm{e}{-4}/1.259\mathrm{e}{-4}$ & 0.681 & $48.45/20.63$ & $2.35\times$ \\
                 & 8 & $7.263\mathrm{e}{-5}/3.127\mathrm{e}{-5}$ & 0.431 & $40.49/12.92$ & $3.13\times$ \\
    \hline
    DBpedia-3072 & 2 & $9.153\mathrm{e}{-2}/1.654\mathrm{e}{-1}$ & 1.807 & $151.76/32.13$ & $4.72\times$ \\
                 & 3 & $2.620\mathrm{e}{-2}/3.742\mathrm{e}{-2}$ & 1.428 & $152.76/33.23$ & $4.60\times$ \\
                 & 4 & $7.145\mathrm{e}{-3}/8.710\mathrm{e}{-3}$ & 1.219 & $154.63/35.16$ & $4.40\times$ \\
                 & 5 & $1.880\mathrm{e}{-3}/2.088\mathrm{e}{-3}$ & 1.111 & $156.10/36.73$ & $4.25\times$ \\
                 & 6 & $5.178\mathrm{e}{-4}/5.101\mathrm{e}{-4}$ & 0.985 & $157.35/37.88$ & $4.15\times$ \\
                 & 7 & $1.849\mathrm{e}{-4}/1.260\mathrm{e}{-4}$ & 0.681 & $158.63/39.16$ & $4.05\times$ \\
                 & 8 & $7.267\mathrm{e}{-5}/3.130\mathrm{e}{-5}$ & 0.431 & $143.89/24.12$ & $5.97\times$ \\
    \hline
  \end{tabular}%
  }
\end{table}

\paragraph{Setup for the retrieval task.}
We evaluate
retrieval at $b\in\{2,\ldots,6\}$ on 100,000 database vectors
from GloVe ($d=200$, padded to $256$) and two DBpedia OpenAI3 embedding sets
($d=1{,}536$ and $3{,}072$, padded to $2{,}048$ and $4{,}096$).  GloVe uses
its 10,000 provided test queries; for each DBpedia set, we use 1,000 disjoint
sampled queries.  Database rows are sampled uniformly without replacement
using seed 1.  Retrieval quality is measured by $1@k$: the fraction of queries whose exact top-1
neighbor appears among the top $k$ results obtained from reconstructed database
vectors.  We report $k\in\{1,2,4,8,16,32,64\}$. 
Retrieval results are averaged over three seeds.

\begin{table}[t]
  \centering
  \caption{$1@k$ retrieval accuracy (in percentage).  Each entry gives
  TurboQuant/DASHQuant; results are averaged over three quantizer seeds.}
  \label{tab:recall}
  \scriptsize
  \setlength{\tabcolsep}{3pt}
  \renewcommand{\arraystretch}{1.4} 
  \resizebox{\linewidth}{!}{%
  \begin{tabular}{lrrrrrrrr}
    \hline
    Dataset & $b$ & $@1$ & $@2$ & $@4$ & $@8$ & $@16$ & $@32$ & $@64$ \\
    \hline
    GloVe-200 & 2 & $55.55/50.19$ & $72.00/66.19$ & $84.55/79.71$ & $92.79/89.20$ & $97.20/94.78$ & $99.05/97.90$ & $99.71/99.12$ \\
              & 3 & $72.98/70.21$ & $88.09/86.18$ & $96.19/95.08$ & $99.02/98.79$ & $99.83/99.78$ & $99.98/99.97$ & $100.00/100.00$ \\
              & 4 & $84.85/83.77$ & $95.80/95.46$ & $99.33/99.27$ & $99.94/99.96$ & $100.00/100.00$ & $100.00/100.00$ & $100.00/100.00$ \\
              & 5 & $92.03/91.66$ & $98.65/98.62$ & $99.94/99.91$ & $100.00/100.00$ & $100.00/100.00$ & $100.00/100.00$ & $100.00/100.00$ \\
              & 6 & $95.66/95.68$ & $99.54/99.62$ & $99.99/100.00$ & $100.00/100.00$ & $100.00/100.00$ & $100.00/100.00$ & $100.00/100.00$ \\
    \hline
    DBpedia-1536 & 2 & $88.40/85.90$ & $97.83/96.77$ & $99.97/99.63$ & $100.00/99.93$ & $100.00/100.00$ & $100.00/100.00$ & $100.00/100.00$ \\
                 & 3 & $93.23/91.63$ & $99.27/99.00$ & $99.97/99.97$ & $100.00/100.00$ & $100.00/100.00$ & $100.00/100.00$ & $100.00/100.00$ \\
                 & 4 & $95.93/95.40$ & $99.77/99.80$ & $99.97/100.00$ & $100.00/100.00$ & $100.00/100.00$ & $100.00/100.00$ & $100.00/100.00$ \\
                 & 5 & $97.43/97.83$ & $99.90/99.90$ & $100.00/99.93$ & $100.00/100.00$ & $100.00/100.00$ & $100.00/100.00$ & $100.00/100.00$ \\
                 & 6 & $98.53/98.70$ & $99.97/100.00$ & $100.00/100.00$ & $100.00/100.00$ & $100.00/100.00$ & $100.00/100.00$ & $100.00/100.00$ \\
    \hline
    DBpedia-3072 & 2 & $91.47/90.90$ & $98.57/98.00$ & $99.83/99.77$ & $100.00/100.00$ & $100.00/100.00$ & $100.00/100.00$ & $100.00/100.00$ \\
                 & 3 & $95.17/95.10$ & $99.27/99.53$ & $100.00/100.00$ & $100.00/100.00$ & $100.00/100.00$ & $100.00/100.00$ & $100.00/100.00$ \\
                 & 4 & $97.77/97.90$ & $99.83/99.90$ & $100.00/100.00$ & $100.00/100.00$ & $100.00/100.00$ & $100.00/100.00$ & $100.00/100.00$ \\
                 & 5 & $98.70/98.67$ & $100.00/99.93$ & $100.00/100.00$ & $100.00/100.00$ & $100.00/100.00$ & $100.00/100.00$ & $100.00/100.00$ \\
                 & 6 & $99.37/99.37$ & $100.00/100.00$ & $100.00/100.00$ & $100.00/100.00$ & $100.00/100.00$ & $100.00/100.00$ & $100.00/100.00$ \\
    \hline
  \end{tabular}%
  }
\end{table}

\paragraph{Results.}
In \cref{tab:nmse-timing}, we present the detailed NMSE values and encode-decode time across all choices of $b = \{2, \cdots, 8\}$. TurboQuant has lower reconstruction error
at low bit widths, but the gap closes as $b$ increases: the
DASHQuant/TurboQuant NMSE ratio decreases from $1.78$--$1.81$ at $b=2$ to
$0.982$--$0.985$ at $b=6$, $0.680$--$0.681$ at $b=7$, and $0.431$ at
$b=8$.  

In \cref{tab:recall}, we present the detailed retrieval quality  across all choices of $b = \{2, \cdots, 6\}$. At $b=2$, TurboQuant's
$1@1$ advantage decreases from $5.36$ percentage points on GloVe to $2.50$
points on DBpedia-$1{,}536$ and $0.57$ points on DBpedia-$3{,}072$.  For
$1@k$, the difference also decreases with $k$: at $b=2$ and $k=64$, the gap
is $0.59$ points on GloVe and zero on both DBpedia datasets.  At $b=5$--$6$,
the methods are effectively tied.



\end{document}